\pdfoutput = 1
\documentclass[11pt]{article}
\usepackage{fullpage}
\usepackage{amsmath}
\usepackage{amsfonts}
\usepackage{amsthm}
\usepackage{natbib}
\usepackage{thm-restate}
\usepackage[pagebackref,linktocpage]{hyperref}
\usepackage{cleveref}

\newtheorem{theorem}{Theorem}[section]
\newtheorem{lemma}[theorem]{Lemma}
\newtheorem{definition}[theorem]{Definition}
\newtheorem{claim}[theorem]{Claim}
\newtheorem{remark}[theorem]{Remark}

\newcommand{\abs}[1]{\left|#1\right|}
\newcommand{\set}[1]{\left\{#1\right\}}

\newcommand{\paren}[1]{\left(#1\right)}

\newcommand{\bR}{{\mathbb R}}
\newcommand{\cR}{{\mathcal R}}
\newcommand{\cC}{{\mathcal C}}
\newcommand{\cD}{{\mathcal D}}

\newcommand{\cX}{{X}}

\newcommand{\cL}{\mathcal{L}}

\newcommand{\one}{{\mathbf 1}}

\newcommand{\sps}[1]{^{({#1})}}

\newcommand{\est}{{\mathsf{est}}}
\newcommand{\prob}{{\mathsf{prob}}}

\newcommand{\sol}{{\mathsf{sol}}}
\newcommand{\opt}{{\mathsf{opt}}}
\newcommand{\cT}{{\mathcal T}}
\newcommand{\genmc}{{\mathsf{GenMC}}}
\newcommand{\ma}{{\mathsf{MA}}}
\newcommand{\grpma}{{\mathsf{GrpMA}}}
\newcommand{\mc}{{\mathsf{MC}}}
\newcommand{\cali}{{\mathsf{Cal}}}
\newcommand{\grpcal}{{\mathsf{GrpCal}}}
\newcommand{\grpmc}{{\mathsf{GrpMC}}}
\newcommand{\grplma}{{\mathsf{GrpLMA}}}
\newcommand{\ber}{{\mathsf{Ber}}}
\DeclareMathOperator*{\E}{\mathbb{E}}

\DeclareMathOperator*{\minimize}{minimize}

\title{Omnipredictors for Constrained Optimization}
\author{Lunjia Hu\thanks{Stanford University. Supported by the Simons Foundation Collaboration on the Theory of Algorithmic Fairness, Omer Reingold’s NSF Award IIS-1908774, and Moses Charikar's Simons Investigators award.} \and 
Inbal Livni-Navon\thanks{Stanford University. Supported by the Simons Foundation Collaboration on the Theory of Algorithmic Fairness, the Sloan Foundation Grant 2020-13941, and the Zuckerman STEM Leadership Program.} \and 
Omer Reingold\thanks{Stanford University. Supported by the Simons Foundation Collaboration on the Theory of Algorithmic Fairness and the Simons Foundation Investigators award 689988.} \and 
Chutong Yang\thanks{Stanford University. Supported by the Simons Foundation Collaboration on the Theory of Algorithmic Fairness and Omer Reingold’s NSF Award IIS-1908774.}
}
\date{}

\allowdisplaybreaks
\begin{document}
\maketitle
\thispagestyle{empty}
\begin{abstract}
The notion of omnipredictors (Gopalan, Kalai, Reingold, Sharan and Wieder ITCS 2021), suggested a new paradigm for loss minimization. Rather than learning a predictor based on a known loss function, omnipredictors can easily be post-processed to minimize any one of a rich family of loss functions compared with the loss of hypotheses in a class $\mathcal C$. It has been shown that such omnipredictors exist and are implied (for all convex and Lipschitz loss functions) by the notion of multicalibration from the algorithmic fairness literature. In this paper, we introduce omnipredictors for constrained optimization and study their complexity and implications. The notion that we introduce allows the learner to be unaware of the loss function that will be later assigned \emph{as well as the constraints that will be later imposed}, as long as the subpopulations that are used to define these constraints are known. We show how to obtain omnipredictors for constrained optimization problems, relying on appropriate variants of multicalibration. We also investigate the implications of this notion when the constraints used are so-called group fairness notions. 
\end{abstract}
\newpage
\tableofcontents
\newpage
\section{Introduction}
A predominant usage for outcome prediction is to inform the choice of a related action. Predicting the probability of a medical condition may help decide on a medical intervention or determine a life insurance premium rate. Predicting the probability of rain may help decide on the method of commuting to work or on a vacation destination or on wedding plans. For each possible action and outcome pair, there may be an associated loss – the cost of catching a cold while riding to work on a bike in the rain or perhaps the cost of changing a wedding venue at the last minute. A learning algorithm may try to come up with a hypothesis that determines an action to minimize an expected loss based on a particular loss function. The challenge in this prevalent paradigm of loss minimization is that different loss functions call for very different learning algorithms, which is problematic for a variety of reasons (e.g.\ multiple relevant loss functions or loss functions that are undetermined at the time of learning). The notion of omnipredictors that was introduced recently by Gopalan, Kalai, Reingold, Sharan and Wieder~\citep{GopalanKRSW22} provides a way to learn a single predictor that can be naturally post-processed (without access to data) to an action that minimizes any one of a very wide collection of loss functions. \citet{GopalanKRSW22} showed that omniprediction is implied by multicalibrated prediction, a notion introduced by Hebert-Johnson, Kim, Reingold and Rothblum in the algorithmic fairness literature \citep{hebert2018multicalibration}.

While loss minimization is a natural goal, it may not be the only consideration in choosing an action. There may, for example, be capacity constraints (e.g.\ a limited number of vaccines) as well as fairness and diversity considerations. In this work, we introduce a notion of omniprediction that applies to the task of loss minimization {\em conditioned on a set of constraints}. For example, imagine we are deciding on which patients would receive a medical intervention when the budget for offering that intervention is limited (capacity constraint), or when we want this intervention to be assigned proportionally to the size of two subpopulations (statistical parity), or when we want the probability of receiving an intervention among patients who experience medical complications
to be the same in two different subpopulations (equal opportunity). Our notion of omniprediction allows learning a single predictor that could be used to minimize a large collection of loss functions, even when arbitrary subsets of constraints are imposed from a rich family of constraints. We show how to formalize such a notion (exposing subtleties not existing in the original notion of omniprediction), how to obtain it using some variants of multicalibration, demonstrating that seeking an accurate depiction of the current world may be useful even when the final goal is a socially engineered action. Finally, we study the interaction between loss minimization and fairness constraints, showing that loss minimization has the potential to support fairness objectives.

\paragraph{Unconstrained Omniprediction.} We assume a distribution $\cD$, over pairs $(x,y)$, where $x\in \cX$ represents an individual, and $y$ represents an outcome associated with $x$. For example, $x$ is the attributes of a patient and $y$ is whether that patient experienced a specific medical condition (in this paper, we will consider Boolean outcomes, i.e., $y\in \{0,1\}$, but the notion could be generalized). We consider individual loss functions. A loss function $\ell$ is applied to an action $a$ and an outcome $y$ and signifies the loss $\ell(y,a)$ incurred when taking action $a$ and observing outcome $y$ (as we will discuss below, our results apply to a more general set of loss functions that can take into account membership of an individual in some predefined subpopulation). 

The learning task of loss minimization is to learn a function $c$ mapping individuals to actions such that the expected loss,  $\E_{(x,y)\sim\cD}[\ell(y,c(x))]$, is at least as small (up to some error term) as $\E_{(x,y)\sim\cD}[\ell(y,c'(x))]$ for any function $c'$ in a hypothesis class $\cC$. Note that different loss functions may require different functions $c$ and different learning algorithms to train them. The notion of omniprediction offers a way for a single algorithm to learn a predictor $p:\cX\rightarrow [0,1]$ that allows optimizing any loss function in a rich family (e.g.\ all loss functions that are convex and $\kappa$-Lipschitz in the action). In this sense, $p$ imitates the true probability predictor $p^*:\cX\rightarrow [0,1]$ where $p^*(x)=\Pr_\cD[y=1\ |\ x]$. Note that for every ``nice" loss function, it is fairly easy to transform $p^*(x)$ to an action $a=\tau_\ell(p^*(x))$ that individually minimizes $\ell(y,a)$ (conditioned on $x$). Loosely, $p$ is an $(\cL, \cC)$-omnipredictor if for every $\ell\in \cL$, applying $\tau_\ell$ to $p$ to get $c(x) = \tau_\ell(p(x))$ minimizes loss $\ell$ compared with the class $\cC$. An omnipredictor resolves the aforementioned disadvantage of traditional loss minimization as it can be trained without knowledge of the specific loss function chosen and the loss function is only needed to decide on an action. 

It has been shown in \citep{GopalanKRSW22} that omniprediction is a somewhat surprising application of the notion of multicalibration, introduced with the motivation of preventing unfair discrimination. Calibration roughly asks that every prediction value be accurate on average over the instances when the prediction value is given. Multicalibration asks a predictor to be calibrated  not only over the entire population but also on many subpopulations (thus, a multicalibrated predictor cannot trade the accuracy of a relevant minority group for the benefit of the majority population). Ignoring some subtleties, a predictor $p$ is  $\cC$-multicalibrated (up to error $\alpha$) if for all $c\in\cC$, 
    $\sum_{v}\abs{\E_{(x,y)\sim\cD}[(y - v)c(x)\one(p(x) = v)]} \le \alpha,$ where the summation is over $v$ in the range of $p$ (we assume the range is finite). It is shown in \citep{GopalanKRSW22} that a $\cC$-multicalibrated is also $(\cal L, \cC)$-omnipredictor for a wide class of loss functions (all convex and Lipschitz loss functions), and \citet{gopalan2022loss} relax the multicalibration requirement to \emph{calibrated multiaccuracy} when the loss functions have additional properties (e.g.\ when they are induced by generalized linear models).

\paragraph{Constraints are Essential but Challenging.}
Omnipredictors constructed in previous work \citep{GopalanKRSW22,gopalan2022loss} allow us to efficiently solve various downstream loss minimization tasks. Each of these tasks aims to minimize the expectation of a loss function and beyond that the solutions to these tasks are not guaranteed to satisfy any non-trivial constraints. However, many loss minimization problems in practice naturally come with constraints that cannot be simply expressed as minimizing an expected loss $\E_{(x,y)\sim \cD}[\ell(y,c(x))]$. 
For example, if an action $c(x)$ represents the amount of resources allocated to individual $x$, it is common to impose a budget constraint $\E[c(x)] \le B$ for an average budget $B$ per individual. Other natural constraints come from the algorithmic fairness literature and are known as group fairness notions. Here, we assume that the entire set $X$ of individuals is partitioned into $t$ subpopulations (i.e., groups) $S_1,\ldots,S_t$. Common examples of group fairness constraints include statistical parity ($\E[c(x)|x\in S_i]$ being approximately equal for every choice of $i=1,\ldots,t$), equal opportunity ($\E[c(x)|x\in S_i,y=1]$ being approximately equal for every $i$), and equalized odds (for every $b=0,1$, the expectation $\E[c(x)|x\in S_i,y=b]$ being approximately equal for every $i$).

Constraints as basic as the budget constraint already impose challenges to the omniprediction results in previous work. This is because in previous work the final action $c(x) = \tau_\ell(p(x))$ is extremely local: it depends only on the loss function $\ell$ and the prediction $p(x)$ for that single individual $x$. Even if $p(x)$ equals the true conditional probability $\Pr_\cD[y = 1|x]$, such local actions that completely ignore the marginal distribution over individuals and the predictions $p(x')$ for other individuals $x'\in X\setminus\{x\}$
cannot satisfy even the simplest budget constraint in general. 
While a loss function can be optimized for every individual separately,
to determine whether an action $c(x)$ would violate the budget constraint, it is necessary to know the actions $c(x')$ assigned to other individuals $x'\in X\setminus\{x\}$.
When constraints are present, omnipredictors are only possible when we allow more flexible ways of turning predictions into actions.

\subsection{Our Contributions}
We start by generalizing the powerful notion of omniprediction to more widely-applicable loss minimization tasks that have constraints.
\paragraph{Defining Omniprediction for Constrained Loss Minimization.}
We consider constrained loss minimization tasks in general forms, where every task has an objective function $f_0:X\times A\times \{0,1\}\to \bR$ and a collection of constraint functions $f_j:X\times A\times \{0,1\}\to \bR$ indexed by $j\in J$. The goal of the task is to find an action function $c:X\to A$ that minimizes the objective $\E_{(x,y)\sim \cD}[f_0(x,c(x),y)]$ while satisfying the constraints $\E_{(x,y)\sim \cD}[f_j(x,c(x),y)]\le 0$ for every $j\in J$. Results in this paper extend to more general tasks where we use an arbitrary Lipschitz function to combine constraints as well as objectives (\Cref{sec:lip}).

Following previous work, for a class $\cT$  of tasks and a class $\cC$ of hypotheses $c:X\to A$, we say a predictor $p:X\to [0,1]$ is an omnipredictor if it allows us to ``efficiently solve'' any task $T\in \cT$ compared to the hypotheses in $\cC$. More specifically, in our constrained setting, an omnipredictor $p$ allows us to ``efficiently produce'' a good action function $c:X\to A$ for any task $T\in \cT$ such that $c$ approximately satisfies all the constraints in $T$, and the objective achieved by $c$ does not exceed (up to a small error) the objective of any $c'\in \cC$ \emph{that satisfy all the constraints of $T$}.

A key challenge in formalizing omniprediction for constrained loss minimization is to specify the procedure of ``efficiently turning'' a predictor $p:X\to [0,1]$ into an action function $c:X\to A$ for a specific task $T\in \cT$. As discussed earlier, previous work only allows $c(x)$ to be $\tau(p(x))$ for a transformation function $\tau$ that only depends on $T$, and this local transformation is not sufficient in our constrained setting. We need more flexible transformations, and we also need to maintain the efficiency of such transformations. We solve this challenge by examining the semantics behind the transformation $\tau(p(x))$ in previous work: this transformation corresponds to solving the task $T$ optimally while pretending that $p(x)$ is the true conditional probability $\Pr_\cD[y=1|x]$. We thus use transformations induced by solving the task on a \emph{simulated distribution} defined by $p$ in our definition of omniprediction (\Cref{def:omni}). 
We show that this not only makes omniprediction possible for constrained problems, but also maintains the efficiency of the transformation.
Moreover, as we discuss below, we can construct omnipredictors for important families of constrained loss minimization problems from group-wise variants of the multiaccuracy and/or multicalibration conditions. Note that conditions such as multiaccuracy and multicalibration are already needed in previous omniprediction results that do not handle constraints!

\paragraph{Constructing Omnipredictors for Group Objectives and Constraints.}
We develop omnipredictors for an important class of constrained loss minimization tasks, namely, tasks with \emph{group objectives and constraints}. Here, as in many problems in the fairness literature, we assume that the set $X$ of individuals is partitioned into $t$ groups $S_1,\ldots,S_t$, and we let $g:X\to [t]$ denote the group partition function, i.e., $g(x) = i$ if and only if $x\in S_i$. We say an objective/constraint function $f:X\times A\times \{0,1\}\to \bR$ is a \emph{group constraint} if there exists $f':[t]\times A\times \{0,1\} \to \bR$ such that $f(x,a,y) = f'(g(x),a,y)$ for every $(x,a,y)\in X\times A\times \{0,1\}$. Tasks with group objectives and constraints are significantly more general than unconstrained tasks in previous work with a loss function $\ell(y,a)$ that does not depend on the individual $x$ at all.

In \Cref{sec:omni}, we show that omnipredictors for loss minimization problems with group objectives and constraints can be obtained from group-wise multiaccuracy and/or multicalibration conditions. Here, group-wise multiaccuracy and multicalibration require the predictor to satisfy multiaccuracy and multicalibration when conditioned on every group $S_i$ (see \Cref{sec:group-ma-mc} for formal definitions). Specifically, we show the following results from the simplest setting to more challenging ones:
\begin{enumerate}
\item We start by considering a simple but general class of objectives/constraints that are \emph{convex and special} (\Cref{def:convex-special}). Objectives in this class include the common $\ell_1$ loss, the squared loss, loss induced by generalized linear models (up to scaling), and group combinations of these loss functions (e.g.\ each group chooses the $\ell_1$ or the squared loss). Constraints in this class include budget constraints and group fairness constraints such as statistical parity, equal opportunity, and equalize odds. In \Cref{thm:squared-linear}, we show that omnipredictors for tasks with convex and special group objectives and constraints can be obtained from group multiaccuracy w.r.t.\ the hypothesis class $\cC$ plus group calibration. This generalizes the results in \citep{gopalan2022loss} to our constrained and multi-group setting.
\item In \Cref{thm:convex-linear}, we show that for general convex and Lipschitz group objectives and constraints, we can construct omnipredictors from group multicalibration w.r.t.\ $\cC$.
This generalizes the results in \citep{GopalanKRSW22} to our constrained and multi-group setting.
\item In \Cref{thm:general}, we show that for general (non-convex) group objectives and constraints, omnipredictors can be obtained from group calibration plus group \emph{level-set} multiaccuracy w.r.t.\ $\cC$, namely,
being accurate in expectation over individuals $x\in S_i$ with $c(x) = a$ for every group $i$, hypothesis $c\in \cC$, and action $a$.
\end{enumerate}

We prove all our omniprediction results in a unified and streamlined fashion using \Cref{lm:approach}. Previously, \citet{gopalan2022loss} also aim to build a unified framework for omnipredictors using the notion of outcome indistinguishability \citep{dwork2021outcome}, but their approach does not fully explain the initial omniprediction result in \citep{GopalanKRSW22}. Our unifying approach not only applies to the results in this paper, but also reconstructs these previous results as unconstrained special cases of our results.

We provide counterexamples in \Cref{sec:counterexamples} to show that it is necessary to strengthen multiaccuracy/multicalibration to their group-wise and occasionally level-set variants in our constrained setting.

\paragraph{Loss Minimization Can Augment Fairness.}
When the constraints imposed are fairness constraints that are aimed to protect one of the subpopulations $S_i$ then loss minimization could support or negate the impact of these constraints. For example, consider a loss function that is one when $x\in S_i$ is awarded a loan and repays it and 0 otherwise. This would incentivize giving loans to members of the protected population that would default on the loan rather than those that would repay it (a similar example can be described in many other domains). This example demonstrates yet again the weakness of so-called group notions of fairness. The weakness of group notion of fairness have been repeatedly demonstrated (cf.\ \citep{dwork2012fairness} for an early example), and often abuses of these notions follow non-monotonicity as in the above example. While we think that $x_1\in S_i$ is more likely to repay the loan than $x_2\in S_i$, we will award the loan to $x_2$. One of our contributions is to show that loss functions with natural monotonicity properties may incentivize a function $c$ that is monotone within each subpopulation, thus strengthening the protection provided by group notions such as statistical parity and equal opportunity. See \Cref{sec:fairness-loss} for formal definitions and more details.

\subsection{Related Work}

Loss minimization under fairness or other constraints is a rich research area. 
For any given fairness definition, it is natural to ask how to learn under the corresponding constraints and how to minimize loss (or maximize utility). This has been studied for various group notions of fairness (cf~\citep{zafar2017fairness}) but also for more refined notions such as metric fairness and multi-group metric fairness~\citep{dwork2012fairness,RothlbumY18,kim2018fairness}. A common approach to combining loss minimization with fairness constraints is to add a fairness regularizer to the risk minimization \citep{donini2018empirical,kamishima2012fairness,zafar2017fairness}. Non-convex constraints have been considered in~\citep{cotter2019optimization}. Accordingly, they also formulate the problem as a non-convex optimization problem which may be hard to solve. There is also a line of empirical work on loss minimization with fairness constraints \citep{zemel2013learning,zafar2017fairnessexperiment, goh2016satisfying}.  Finally, some recent related works focus on other learning setting under fairness constraint, like learning policies~\citep{nabi2019learning}, online learning~\citep{bechavod2022individually}, federated learning~\citep{hu2022provably}, and ranking~\citep{dwork2019learning}.

A key difference between our work and most previous work on loss minimization is that we aim for learning a single predictor that can efficiently solve a variety of downstream constrained loss minimization tasks. Moreover, as we do not make any assumption on the true data distribution $\cD$, we consider it infeasible to learn the distribution $\cD$ entirely and we only require conditions such as multicalibration that can be much easier to achieve using existing algorithms in the literature.
Some works, such as \citep{celis2019classification,agarwal2018reductions,narasimhan2018learning,sharifi2019average}, can deal with multiple loss minimization tasks but they require approximately learning the true distribution $\cD$ within a small total variation distance or approximately learning the true labels. %

In an influential paper, Hardt, Price and Srebro \citep{hardt2016equality} propose equalized odds and equal opportunity as group notions of fairness. They give methods of post-processing a predictor to enforce these constraints while minimizing loss. They show optimality compared with solutions that can be obtained from post-processing the predictor, whereas in this work we directly aim for optimality with respect to a rich pre-specified hypothesis class $\cC$. We consider more general loss functions with real-valued actions compared to the loss functions in \citep{hardt2016equality} that only take binary values as input, and we also consider more general constraints beyond the group fairness constraints in \citep{hardt2016equality}.

\citet{rothblum2021multi} use the notion of outcome indistinguishability~\citep{dwork2021outcome}, closely related to multicalibration, to obtain loss minimization, not only on the entire population but also on many subpopulations. 
Their approach relies on a locality property of the loss function which they term \emph{$f$-proper}. When this property is satisfied, for every fixed individual $x_0\in X$, the optimal action $c(x_0)$ for that individual $x_0$ only depends on $\E[y|x=x_0]$ and not on $\E[y|x = x_1]$ for other individuals $x_1\in X\setminus\{x_0\}$. In our constrained setting, this locality property fails to hold: to satisfy a group constraint, the action $c(x_0)$ must coordinate with the actions $c(x_1)$ for other individuals $x_1$ in or out of the group/subpopulation of $x_0$.

Independently of our work, \citet{globus2022multicalibrated} also study the problem of solving downstream tasks by post-processing multicalibrated predictors. They 
focus on the 0-1 loss for classification tasks and thus their results do not imply the full power of omnipredictors that handle arbitrary loss functions from a rich family.
They also focus on a few specific group fairness constraints, whereas we consider more general classes of constraints.
By assuming multicalibration with respect to delicately-designed classes, their predictors can be efficiently post-processed to satisfy constraints on \emph{intersecting} groups. 
Again independently of our work, \citet{kim2022making} study omniprediction in an (unconstrained) performative setting, where the distribution of the outcome $y$ of an individual $x$ can change based on the action $c(x)$.

\section{Problem Setup}
\label{sec:setup}
Throughout the paper, we use $X$ to denote a non-empty set of individuals, and use $\cD$ to denote a distribution over $X\times \{0,1\}$. We use $A$ to denote a non-empty set of actions, and use $c:X\to A$ to denote an action function that assigns an action $c(x)$ to every individual $x\in X$ (e.g.\ hiring the individual or not).
We occasionally consider a \emph{randomized} action function $c:X\to \Delta_A$ that assigns every individual $x\in X$ a \emph{distribution} $c(x)\in \Delta_A$ over actions in $A$. For generality we sometimes only make statements about randomized action functions, where one should view a deterministic action function $c:X\to A$ as the randomized action function $c':X\to \Delta_A$ where $c'(x)\in \Delta_A$ is the degenerate distribution supported on $c(x)$ for every $x\in X$.
\subsection{Constrained Loss Minimization Tasks}
Given a loss function $f_0:X\times A\times \{0,1\} \to \bR$ and a collection of constraints $f_j:X\times A\times \{0,1\} \to \bR$ indexed by $j\in J$, we define a constrained loss minimization task $T$ to be the following optimization problem:
\begin{align}
\minimize_{c:X\to A} & \E_{(x,y)\sim \cD}f_0(x,c(x),y)\label{eq:target-optimization}\\
\text{s.t.} & \E_{(x,y)\sim \cD}f_j(x,c(x),y) \le 0 \quad\text{for every }j\in J. \nonumber 
\end{align}
It is often challenging to solve a task $T$ optimally, and we need to consider approximate and potentially randomized solutions. 
For $\beta\in \bR$ and $\varepsilon\in \bR_{\ge 0}$, we define $\sol_\cD(T,\beta,\varepsilon)$ to be the set of randomized action functions $c:X\to \Delta_A$ satisfying 
\begin{align*}
& \E_{(x,y)\sim \cD}\E_{a\sim c(x)}f_0(x,a,y) \le \beta, \quad\text{and}\\
& \E_{(x,y)\sim \cD}\E_{a\sim c(x)}f_j(x,a,y) \le \varepsilon \quad \text{for every }j\in J.
\end{align*}
 For a class $\cC$ of functions $c:X\to \Delta_A$, we define 
\[
\opt_\cD(T,\cC,\varepsilon) := \inf\{\beta\in \bR:\cC\cap \sol_\cD(T,\beta,\varepsilon)\ne \emptyset\}.
\]
Note that $\opt_\cD(T,\cC,\varepsilon)$ may take any value in $\bR\cup \{\pm \infty\}$, where we define $\inf \emptyset = +\infty$.
In \Cref{sec:lip}, we show how results in this paper extend to more general tasks where we combine constraints and objectives using arbitrary Lipschitz functions.
\subsection{Omnipredictors for Constrained Loss Minimization}
An omnipredictor, as introduced by \citet{GopalanKRSW22}, allows us to solve a class of downstream loss minimization tasks without training a different model from scratch for every task in the class. Previous work focuses on omnipredictors for unconstrained loss minimization \citep{GopalanKRSW22,gopalan2022loss}. We generalize this notion to constrained loss minimization in the following definition. For a distribution $\cD$ over $X\times\{0,1\}$ and a predictor $p:X\to [0,1]$, we define the simulated distribution $\cD_p$ to be the distribution of $(x,y')\in X\times\{0,1\}$ where we first draw $(x,y)$ from $\cD$ and then draw $y'$ from the Bernoulli distribution $\ber(p(x))$ with mean $p(x)$.
\begin{definition}
\label{def:omni}
Let $\cD$ be a distribution over $X\times \{0,1\}$ and $\varepsilon \ge 0$ be a parameter.
Let $\cT$ be a collection of constrained loss minimization tasks and let $p:X\to [0,1]$ be a predictor. For classes $\cC,\cC_p$ of functions $c:X\to \Delta_A$, we say $p$ is a $(\cT,\cC,\cC_p,\varepsilon)$-omnipredictor on $\cD$ if the following holds for any $T\in \cT$. Assuming $\beta^*:= \opt_\cD(T,\cC,0)\in \bR$ and $\beta:= \opt_{\cD_p}(T,\cC_p,\varepsilon/3)\in \bR$, we have
\[
\cC_p\cap \sol_{\cD_p}(T,\beta + \varepsilon/3,2\varepsilon/3) \subseteq \sol_\cD(T,\beta^* + \varepsilon,\varepsilon).
\]
\end{definition}
Suppose we have an omnipredictor $p$ as in the definition above, and we want to solve an arbitrary constrained loss minimization task $T\in \cT$ in comparison with the class $\cC$, i.e., we want to find a solution in $\sol_\cD(T,\beta^* + \varepsilon,\varepsilon)$. Instead of collecting data points from $\cD$ and solve the task from scratch, we just need to find a solution in $\cC_p\cap \sol_{\cD_p}(T,\beta + \varepsilon/3,2\varepsilon/3)$, i.e., a solution $c\in \cC_p$ that approximately solves the task on the simulated distribution $\cD_p$. This is usually much easier than solving the task on the original distribution $\cD$ for the following reasons. First, since we know $p$, we know the conditional distribution of $y$ given $x$ in $(x,y)\sim \cD_p$, and thus the only unknown part about $\cD_p$ is the marginal distribution of $x$, which can be learned from \emph{unlabeled} data drawn from $\cD$ (i.e., examples of $x$ in $(x,y)\sim \cD$ with $y$ concealed). Secondly, the sample and computational complexity of solving the new task depends on $\cC_p$ instead of $\cC$. In all of the omniprediction results in this paper, we choose $\cC_p$ to be very simple (as in \Cref{def:c-p}) so that its complexity depends on the size of the range of $p$, which can be made to be very small ($O(1/\varepsilon)$), whereas $\cC$ can be significantly more complex. Specifically, every function in $\cC_p$ assigns the same action (or same distribution over actions) to individuals $x$ in the same subpopulation group with the same $p(x)$. In \Cref{sec:algo-simulated} we give very efficient algorithms for solving constrained loss minimization tasks given omnipredictors.

In previous work on omniprediction where there are no constraints,
the optimal solution $c$ on the simulated distribution $\cD_p$ is trivial to find: it is given by choosing $c(x)$ so that $\E_{y\sim \ber(p(x))}f_0(x,c(x),\allowbreak y)$ is minimized (Bayes optimal solution). That is, the optimal $c(x)$ depends only on $x, f_0$, and $p(x)$ (often $f_0$ does not depend on $x$ and thus $c(x)$ only depends on $f_0$ and $p(x)$). Because of this locality property, previous definitions of omniprediction for \emph{unconstrained} loss minimization simply explicitly uses the optimal solution on the simulated distribution $\cD_p$ without defining a task on $\cD_p$ or even without defining $\cD_p$ at all. Our \Cref{def:omni} not only generalizes these previous definitions, but also deals with more challenging tasks with constraints where the locality property fails to hold.

\subsection{Group Multiaccuracy and Multicalibration}
\label{sec:group-ma-mc}
A main contribution of this paper is showing that omnipredictors for a variety of constrained loss minimization problems can be obtained from group-wise multiaccuracy and/or multicalibration conditions. The notions of multiaccuracy and multicalibration are introduced by \citet{hebert2018multicalibration} and \citet{kim2019multiaccuracy}, and there are many algorithms for achieving these notions in previous work (see \Cref{sec:algo-ma-mc}). We define these notions here as special cases of the following generalized multicalibration notion. For the definitions below, we assume $\cD$ is a distribution over $X\times\{0,1\}$ and $\varepsilon \ge 0$ is a parameter.
\begin{definition}[Generalized multicalibration ($\genmc$) (see e.g.\ {\citep[Definition 1.1 in Supplementary Information]{doi:10.1073/pnas.2108097119}})] Let $W$ be a class of functions $w:X\times [0,1]\to \bR$. We say a predictor $p:X\to [0,1]$ satisfies $(W,\varepsilon)$-generalized multicalibration w.r.t.\ distribution $\cD$ if
\[
\left|\E_{(x,y)\sim \cD}[(y - p(x))w(x,p(x))]\right|\le\varepsilon \ \ \text{for every }w\in W.
\]
For simplicity, we additionally require the range of $p$, $\mathsf{range}(p):=\{p(x):x\in X\}$, to be a \emph{finite} subset of $[0,1]$.
We use $\genmc_\cD(W, \varepsilon)$ to denote the set of predictors $p$ satisfying the conditions above.
\end{definition}
We define multiaccuracy and multicalibration below as special cases of $\genmc$ in a general group-wise setting. Here, we assume that the set $X$ of individuals is partitioned into $t$ groups (i.e., subpopulations). We use $g:X\to [t]$ to denote the group partition function that assigns every individual $x\in X$ a group index $g(x) \in [t]:=\{1,\ldots,t\}$.
\begin{definition}[Group Multiaccuracy ($\grpma$)]
\label{def:grpma}
For a class $H$ of functions $h:X\to \bR$, we define the set $\grpma_{\cD}(H,g,\varepsilon)$ of $(H,g,\varepsilon)$-multiaccurate predictors $p$ w.r.t.\ distribution $\cD$ to be $\genmc_\cD(W,\varepsilon)$, where $W$ consists of all functions $w:X\times[0,1]\to \bR$ such that there exist $h\in H$ and $\tau:[t]\to [-1,1]$ satisfying $w(x,v) = h(x)\tau(g(x))$ for every $(x,v)\in X\times[0,1]$.
Equivalently, $\grpma_\cD(H,g,\varepsilon)$ is the set of predictors $p:X\to [0,1]$ satisfying the following for every $h\in H$:
\[
\sum_{i\in [t]}\left|\E_{(x,y)\sim \cD}[(y - p(x))h(x)\one(g(x) = i)]\right| \le \varepsilon.
\]
Here $\one(\cdot)$ is the 0-1 indicator function. When the distribution $\cD$ is clear from context, we often drop it and write $\grpma(H,g,\varepsilon)$ (similarly for other definitions below).
\end{definition}
The equivalence in the definition above and other definitions below in this section are straightforward to prove. We include a proof in \Cref{sec:equivalence} for completeness.

\begin{definition}[Group Multicalibration ($\grpmc$)]
For a class $H$ of functions $h:X\to \bR$, we define the set $\grpmc_{\cD}(H,g,\varepsilon)$ of $(H,g,\varepsilon)$-multicalibrated predictors $p$ w.r.t.\ distribution $\cD$ to be $\genmc_\cD(W,\varepsilon)$, where $W$ consists of all functions $w:X\times[0,1]\to \bR$ such that there exist $h\in H$ and $\tau:[t]\times [0,1]\to [-1,1]$ satisfying $w(x,v) = h(x)\tau(g(x), v)$ for every $(x,v)\in X\times[0,1]$. Equivalently, $\grpmc_{\cD}(H,g,\varepsilon)$ is the set of predictors $p:X\to [0,1]$ satisfying the following for every $h\in H$:
\begin{align*}
\sum_{i\in [t]}\sum_{v\in \mathsf{range}(p)}\left|\E_{(x,y)\sim \cD}[(y - p(x))h(x)\one(g(x) = i,p(x) = v)]\right| \le \varepsilon.
\end{align*}
\end{definition}

The following definition of group calibration is a special case of group multicalibration where $H$ only contains the constant function $h$ that maps every $x\in X$ to $1$:
\begin{definition}[Group Calibration ($\grpcal$)]
We define the set $\grpcal_\cD(g,\varepsilon)$ of $(g,\varepsilon)$-calibrated predictors $p$ w.r.t\ distribution $\cD$ to be $\genmc_\cD(W,\varepsilon)$, where $W$ consists of all functions $w:X\times[0,1]\to [-1,1]$ such that there exists $\tau:[t]\times[0,1]\to [-1,1]$ satisfying $w(x,v) = \tau(g(x),v)$ for every $(x,v)\in X\times [0,1]$.
\end{definition}
The following definition is a variant of group multiaccuracy where $\tau$ 
the transformation $\tau$ also takes the function value $h(x)$ as input, and 
we view individuals $x$ with the same $h(x)$ as belonging to the same level set of $h$.
\begin{definition}[Group Level-Set Multiaccuracy ($\grplma$)]
For an arbitrary finite set $A$ and a class $H$ of functions $h:X\to A$,
we define the set $\grplma_\cD(H,g,\varepsilon)$ of predictors $p$ satisfying $(H,g,\varepsilon)$-level-set multiaccuracy w.r.t\ distribution $\cD$ to be $\genmc_\cD(W,\varepsilon)$, where $W$ consists of all functions $w:X\times[0,1]\to [-1,1]$ such that there exist $h\in H$ and $\tau:[t]\times A\to [-1,1]$ satisfying $w(x,v) = \tau(g(x),h(x))$ for every $(x,v)\in X\times [0,1]$.
Equivalently, $\grplma(H,g,\varepsilon)$ is the set of predictors $p:X\to [0,1]$ satisfying the following for every $h\in H$:
\[
\sum_{i\in [t]}\sum_{a\in A}\left|\E_{(x,y)\sim \cD}[(y - p(x))\one(g(x) = i,h(x) = a)]\right| \le \varepsilon.
\]
\end{definition}
When the group partition function $g$ is a constant function $g_0$ that assigns every individual to the same group, we recover notions in the standard single-group setting: multiaccuracy ($\ma_\cD(H,\varepsilon):= \grpma_\cD(H,g_0,\varepsilon)$), multicalibration ($\mc_\cD(H,\varepsilon):= \grpmc_\cD(H,g_0,\varepsilon)$), and calibration ($\cali_\cD(\varepsilon):= \grpcal_\cD(g_0,\varepsilon)$).
\section{Our Approach}
We describe our general approach for constructing and analyzing omnipredictors for constrained loss minimization tasks. Our approach is similar in spirit to the outcome indistinguishability perspective taken by \citep{gopalan2022loss}, but our approach is more general: it takes constraints into account and can also be applied to reconstruct the results in previous papers on omnipredictors \citep{GopalanKRSW22,gopalan2022loss}. In particular, we overcome the limitation of \citep{gopalan2022loss} that it falls short of fully explaining the initial omnipredictors results in \citep{GopalanKRSW22}.
Our approach is based on the following key lemma:
\begin{restatable}{lemma}{lmapproach}
\label{lm:approach}
Let $\cD$ be a distribution over $X\times \{0,1\}$ and $\varepsilon \ge 0$ be a parameter.
Let $\cT$ be a collection of constrained loss minimization tasks and let $\cC,\cC_p$ be classes of functions $c:X\to \Delta_A$.
If a predictor $p$ satisfies the following two properties for every $T\in \cT$, then $p$ is a $(\cT,\cC,\cC_p,\varepsilon)$-omnipredictor on $\cD$:
\begin{enumerate}
\item Let $f_0$ be the loss function of $T$ and $(f_j)_{j\in J}$ be the constraints of $T$. For every $c\in \cC$, there exists $c'\in \cC_p$ such that for every $j\in \{0\}\cup J$,
\begin{equation}
\E_{(x,y)\sim \cD_p}\E_{a\sim c'(x)}f_j(x,a,y)
\le \E_{(x,y)\sim \cD}\E_{a\sim c(x)}f_j(x,a,y) + \varepsilon/3.\label{eq:approach-1}
\end{equation}
\item For every $c\in \cC_p$ and every $j\in \{0\}\cup J$,
\begin{equation}
\E_{(x,y)\sim \cD}\E_{a\sim c(x)}f_j(x,a,y)
\le \E_{(x,y)\sim \cD_p}\E_{a\sim c(x)}f_j(x,a,y) + \varepsilon/3.\label{eq:approach-2}
\end{equation}
\end{enumerate}
\end{restatable}
\Cref{lm:approach} reduces the task of constructing an omnipredictor to satisfying the conditions in \eqref{eq:approach-1} and \eqref{eq:approach-2}.
We prove  \Cref{lm:approach} in \Cref{sec:reconstruct} and show how to apply it 
to construct omnipredictors for a variety of constrained loss minimization tasks in \Cref{sec:omni}. \Cref{lm:approach} allows us to give short and streamlined proofs for all our results in \Cref{sec:omni}, and these results generalize previous results in \citep{GopalanKRSW22,gopalan2022loss} as special cases.
\section{Omnipredictors from Group Multiaccuracy and Multicalibration}
\label{sec:omni}
In this section, we apply \Cref{lm:approach} and show that we can obtain omnipredictors for loss minimization tasks with group objectives and constraints from group multiaccuracy and/or multicalibration conditions. Here, we assume that the individual set $X$ is partitioned into $t$ groups by a group partition function $g:X\to [t]$ assigning a group index $g(x)\in [i]$ to every individual $x\in X$.
\begin{definition}
\label{def:group-constraint}
We say an objective/constraint function $f:X\times A\times \{0,1\}\to \bR$ is a group objective/constraint if there exists $f':[t]\times A \times \{0,1\}\to \bR$ such that $f(x,a,y) = f'(g(x),a,y)$ for every $(x,a,y)\in X\times A\times \{0,1\}$.
\end{definition}

Proofs for the results in this section are deferred to \Cref{sec:proof-omni}. These results show that algorithms in previous work for achieving multiaccuracy and multicalibration allow us to obtain omnipredictors even when constraints are imposed on the loss minimization tasks. We discuss these algorithms in more detail in \Cref{sec:algo-ma-mc}.

We start with a basic case where the objectives and constraints are convex and special, defined below. We use $\partial f(x,a)$ to denote $f(x,a,1) - f(x,a,0)$. 
\begin{definition}
\label{def:convex-special}
Let the action set $A\subseteq\bR$ to be an interval.
We say an objective/constraint function $f:X\times A\times \{0,1\}\to \bR$ is convex if $f(x,\cdot,y)$ is convex for every fixed $(x,y)\in X\times \{0,1\}$. We say $f$ is special if there exist $\tau_1,\tau_2:[t]\to [-1,1]$ such that $\partial f(x,a) = \tau_1(g(x)) + \tau_2(g(x))a$.
\end{definition}

Examples of convex and special group objectives when $A = [0,1]$ include the $\ell_1$ loss $f(x,a,y) = |a - y|/2$, the squared loss $f(x,a,y) = (a-y)^2/2$, and group-wise combinations of them (every group chooses either $\ell_1$ or squared loss).
As demonstrated in \citep{gopalan2022loss}, loss functions induced from generalized linear models are also special after appropriate scaling.
Examples of convex and special constraints include all \emph{linear constraints}, i.e., constraint functions $f$ for which there exist $\tau_1,\tau_2:[t]\to \bR$ and $\tau_3:[t]\to [-1,1]$ such that
\begin{equation}
\label{eq:linear-constraint}
f(x,a,y) = \tau_1(g(x)) + \tau_2(g(x))a + \tau_3(g(x))ay
\end{equation}
for every $(x,a,y)\in X\times A\times\{0,1\}$.
Linear constraints are general enough to express fairness constraints such as statistical parity, equal opportunity (equal true positive rates), and equalized odds (equal true positive rates and equal false positive rates) as follows. For every group $i\in [t]$, define $r_i:= \Pr[g(x) = i]$, $r_i^+:= \Pr[g(x)=i|y=1]$, and $r_i^-:=\Pr[g(x) = i|y=0]$. These fairness constraints can be expressed as\footnote{Here we assume that we know $r_i,r_i^+,r_i^-$ for simplicity. These quantities can be estimated from unlabeled data and a predictor satisfying group calibration.}
\begin{align*}
\E[\one(g(x) = i)c(x)] & = r_i\E[c(x)], \tag{statistical parity}\\
\E[\one(g(x) = i)c(x)y] & =r_i^+\E[c(x)y], \tag{equal true positive rates}\\
\E[\one(g(x) = i)c(x)(1-y)] & =r_i^-\E[c(x)(1-y)].
\tag{equal false positive rates}
\end{align*}
Each of the above fairness constraints can be written as $\E[f(x,c(x),y)] = 0$ for an appropriate $f$ satisfying \eqref{eq:linear-constraint}. For example, for statistical parity, we choose $f$ as follows:
\begin{align*}
f(x,a,y) = \one(g(x) = i)a - r_ia. \tag{statistical parity}
\end{align*}
Moreover, we can express \emph{approximate} fairness constraints as a combination of linear constraints because $|\E[f(x,c(x),y)]| \le \alpha$ is equivalent to $\E[f(x,c(x),y)- \alpha] \le 0$ and $\E[-f(x,c(x),y)-\alpha]\le 0$.

For tasks with group objectives/constraints, we often choose the class $\cC_p$ in our definition of omnipredictors (\Cref{def:omni}) as in the following definition:
\begin{definition}
\label{def:c-p}
For an action set $A$, a group partition function $g:X\to [t]$ and a predictor $p:X\to [0,1]$, we define $\cC_p(g)$ to be the class consisting of all functions $c:X\to A$ such that there exists $\tau:[t]\times[0,1]\to A$ satisfying $c(x) = \tau(g(x),p(x))$ for every $x\in X$. We define $\cC_p^{\mathsf {rand}}(g)$ to be the class consisting of all functions $c:X\to \Delta_A$ such that there exists $\tau:[t]\times[0,1]\to \Delta_A$ satisfying $c(x) = \tau(g(x),p(x))$ for every $x\in X$. 
\end{definition}
We now state our omniprediction theorem for convex and special constraints and objectives. In the theorems below, we use $\cD$ to denote an underlying distribution over $X\times \{0,1\}$ and use $\cC$ to denote a class of functions $c:X\to A$.
\begin{restatable}{theorem}{thmsquaredlinear}
\label{thm:squared-linear}
Let $A=[0,1]$ be an action set.
Let $\cT$ be a class of tasks that only have group constraints and group objectives that are all convex and special. 
Let $p$ be a predictor in $\grpma_\cD(\cC,g,\allowbreak \varepsilon/6)\cap \grpcal_\cD(g, \varepsilon/6)$ and define $\cC_p(g)$ as in \Cref{def:c-p}. Then $p$ is a $(\cT,\cC,\cC_p(g),\varepsilon)$-omni\-predictor on $\cD$.
\end{restatable}
We remark that the convexity assumption in the theorem above can be removed if we replace $\cC_p(g)$ with $\cC_p^{\mathsf{rand}}(g)$ (\Cref{thm:variant}).
Once we construct an omnipredictor using \Cref{thm:squared-linear} (and other theorems in this section), we can efficiently transform it into nearly optimal actions for any task $T\in \cT$ using a small amount of unlabeled data from $\cD$ (see \Cref{sec:algo-simulated}).
\Cref{thm:squared-linear} generalizes the results in \citep{gopalan2022loss} that hold in the single-group unconstrained setting. Our following theorem deals with general convex and Lipschitz group objectives and constraints and it generalizes the results in \citep{GopalanKRSW22}. 
\begin{definition}
\label{def:lipschitz}
We say an objective/constraint function $f:X\times A\times \{0,1\}\to \bR$ is $\kappa$-Lipschitz if $f(x,\cdot,y)$ is $\kappa$-Lipschitz for every fixed $(x,y)\in X\times \{0,1\}$. We say $f$ has $B$-bounded difference if $\partial f(x,a)\in [-B,B]$ for every $(x,a)\in X\times A$.
\end{definition}
\begin{restatable}{theorem}{thmconvexlinear}
\label{thm:convex-linear}
Let $A=[0,1]$ be an action set.
Let $\cT$ be a class of tasks that only have group objectives and group constraints that are all convex and $1$-Lipschitz and have $1$-bounded differences. 
Let $p$ be a predictor in $\grpmc_\cD(\cC, g, \varepsilon/15)\cap \grpcal_\cD(g,\varepsilon/15)$ and define $\cC_p(g)$ as in \Cref{def:c-p}. Then $p$ is a $(\cT,\cC,\cC_p(g),\varepsilon)$-omnipredictor on $\cD$.
\end{restatable}

Finally, we consider general group constraints. These constraints allows us to constrain the entire distribution of $c(x)$ (e.g.\ constraints on $\Pr[c(x)\in A']$ for $A'\subseteq A$) and the distribution of $c(x)$ within each group (e.g.\ constraints on $\Pr[c(x)\in A', g(x) = i]$).
\begin{restatable}{theorem}{thmgeneral}
\label{thm:general}
Let $A$ be a finite non-empty action set. Let $\cT$ be a class of tasks with group constraints and group objectives that all have $1$-bounded differences. 
Let $p$ be a predictor in $\grplma_\cD(\cC, \allowbreak g, \varepsilon/3)\cap \grpcal_\cD(g, \varepsilon/3)$ and define $\cC_p^{\mathsf{rand}}(g)$ as in \Cref{def:c-p}. Then $p$ is a $(\cT,\cC,\cC_p^{\mathsf{rand}}(g),\varepsilon)$-omnipredictor on $\cD$.
\end{restatable}
We give counterexamples in \Cref{sec:counterexamples} showing that strengthening standard multiaccuracy and multicalibration to their group-wise and/or level-set variants in the theorems above is necessary.

\section{Interaction between Group Fairness and Loss Minimization}\label{sec:fairness-loss}
In this section we explain how we can use our omnipredictors to get an additional property, which we call rank-preserving. The intuition is that if we assume the predictor $p:X\rightarrow[0,1]$ describes an approximation to the true probability $\Pr_{(x,y)\sim\cD}[y=1]$, then we want individuals $x$ with higher $p(x)$ to get higher action values, for real-valued actions $A\subseteq[0,1]$. This requirement can be thought of as  a fairness property, that individuals that are more likely to succeed (within the same group) should get higher actions.

\begin{definition}
    A transformation $\tau:[t]\times[0,1]\rightarrow A$ is called rank-preserving if for all $i\in[t]$ and $v>v'\in [0,1]$ we have $\tau(i,v)\geq\tau(i,v')$.
\end{definition}

We determine when we can choose the transformation $\tau$ applied to the omnipredictor $p$ to be rank-preserving. Our first observation is that the loss function $f_0$ should also be rank-preserving, i.e. if $a>a'$, then $f_0(i,a,1) \leq f_0(i,a',1)$ (and vice versa for $0$). If $f_0$ is the distance between $a$ and $y$, it satisfy the property. We require the predictor to be monotone, $\forall v>v'$, $\E_{(x,y)\sim\cD}[y|p(x)=v]\geq\E_{(x,y)\sim\cD}[y|p(x)=v'] $.

For the case of a single linear constraint per group $g(i)$, we prove that for any omnipredictor, there is always an optimal transformation that is rank-preserving. 
\begin{lemma}
    \label{lem:det-rank-p}
    Let $A=[0,1]$ be the action set, $\cT$ be class of tasks with linear constraints. Assume that for every task $T\in\cT$ the objective function is rank-preserving and convex, and that for every group $i\in[t]$, there is only a single constraint $f_j$ (expressed as in \eqref{eq:linear-constraint}) in which $\tau_1(i),\tau_2(i),\tau_3(i)\neq 0$, and $\tau_2(i)\tau_3(i)\geq 0$. Then for every monotone omnipredictor $p$ we have 
\[
    \opt_\mathcal{D}(T,\text{rank-preserving }\cC_p(g),\varepsilon)=\opt_\mathcal{D}(T,\cC_p(g),\varepsilon).
\]
\end{lemma}
We remark that the requirement $\tau_2(i)\tau_3(i)\geq 0$ is necessary. A constraint with opposite signs can encourage having $a=1$ for individuals $(x,y),g(x)=i$, but discourage $a=1$ for those with $(x,y),g(x)=i,y=1$. It is not possible to have a rank-preserving transformation under such constraint, and this emphasize that both the constraints and the loss functions should be appropriately chosen.

For the more general case of random transformations, and multiple constraints per group $i$, we prove a similar lemma for outcome-oblivious constraints, i.e., constraints $f_j$ that do not depend on $y$ (e.g.\ budget / statistical parity constraints).
\begin{lemma}
    \label{lem:rand-rank-p}
    Let $A\subseteq[0,1]$ be a discrete action set, $\cT$ be class of tasks with constraints that are independent of the outcome.  Then for a monotone omnipredictor $p$ we have 
    \[ \opt_\mathcal{D}(T,\text{rank-preserving }\cC_p^{\mathsf{rand}}(g),\varepsilon)=\opt_\mathcal{D}(T,\cC_p^{\mathsf{rand}}(g),\varepsilon).\]
\end{lemma}

 Proofs and more details are in \Cref{sec:rank-p}.
\bibliography{ref}
\newpage
\appendix
\onecolumn
\section{Proof of Equivalence in Multiaccuracy and Multicalibration Definitions}
\label{sec:equivalence}
We prove the equivalence relationship in \Cref{def:grpma}. Similar proofs can be applied to other definitions in \Cref{sec:group-ma-mc}.
\begin{claim}
In \Cref{def:grpma}, a predictor $p$ belongs to $\grpma(\cC,g,\varepsilon)$ if and only if
\begin{equation}
\label{eq:equivalence}
\sum_{i\in [t]}|\E_{(x,y)\sim \cD}[(y - p(x))h(x)\one(g(x) = i)]| \le \varepsilon \quad \text{for every }h\in H.
\end{equation}
\end{claim}
\begin{proof}
We first show that $p\in \grpma(\cC,g,\varepsilon)$ implies \eqref{eq:equivalence}. For a fixed $h\in H$, we choose $\tau:[t]\to [-1,1]$ such that 
\begin{equation}
\label{eq:equivalence-1}
\tau(i) = \mathsf{sign}\left(\E_{(x,y)\sim \cD}[(y - p(x))h(x)\one(g(x) = i)]\right),
\end{equation}
where $\mathsf{sign}(v) = 1$ if $v \ge 0$, and $\mathsf{sign}(v) = -1$ if $v < 0$.
By our assumption $p\in \grpma(\cC,g,\varepsilon)$,
\begin{align*}
\varepsilon & \ge \E_{(x,y)\sim \cD}[(y - p(x))h(x)\tau(g(x))]\\
& = \sum_{i\in [t]}\E_{(x,y)\sim \cD}[(y - p(x))h(x)\one(g(x) = i)\tau(i)]\\
& = \sum_{i\in [t]}|\E_{(x,y)\sim \cD}[(y - p(x))h(x)\one(g(x) = i)]|.\tag{by \eqref{eq:equivalence-1}}
\end{align*}
This proves \eqref{eq:equivalence}. Now we prove that \eqref{eq:equivalence} implies $p\in \grpma(\cC,g,\varepsilon)$. For any $h\in H$ and $\tau:[t]\to [-1,1]$, 
\begin{align*}
& \E_{(x,y)\sim \cD}[(y - p(x))h(x)\tau(g(x))]\\
= {} & \sum_{i\in [t]}\E_{(x,y)\sim \cD}[(y - p(x))h(x)\one(g(x) = i)\tau(i)]\\
\le {} & \sum_{i\in [t]}|\E_{(x,y)\sim \cD}[(y - p(x))h(x)\one(g(x) = i)]| \tag{by $\tau(i)\in [-1,1]$}\\
\le {} & \varepsilon. \tag{by \eqref{eq:equivalence}}
\end{align*}
This proves $p\in \grpma(\cC,g,\varepsilon)$.
\end{proof}
\begin{remark}
The proof above can be adapted to show that if we restrict $\tau$ to only output values in $\{-1,1\}$ instead of $[-1,1]$, we also get an equivalent definition of $\grpma$, and this holds for other definitions in \Cref{sec:group-ma-mc} as well.
\end{remark}
\section{Proof of Lemma~\ref{lm:approach}}
\label{sec:reconstruct}
We restate and prove \Cref{lm:approach} below.
\lmapproach*
\begin{proof}
Fix an arbitrary task $T\in \cT$. Define $\beta^* := \opt_\cD(T,\cC,0)$ and $\beta:= \opt_{\cD_p}(T,\cC_p,\varepsilon/3)$ as in \Cref{def:omni}. By the definition of $\beta^*$, for any $\beta'> \beta^*$, there exists $c\in \cC\cap \sol_\cD(T,\beta',0)$. By \eqref{eq:approach-1}, there exists $c'\in \cC_p\cap \sol_{\cD_p}(T,\beta' + \varepsilon/3,\varepsilon/3)$. This implies that $\beta \le \beta' + \varepsilon/3$, and thus $\beta \le \beta^* + \varepsilon/3$. Now we have $\beta + \varepsilon/3 \le \beta^* + 2\varepsilon/3$, and thus
\begin{equation}
\label{eq:approach-proof-1}
\cC_p\cap \sol_{\cD_p}(T,\beta + \varepsilon/3,2\varepsilon/3) \subseteq \cC_p\cap \sol_{\cD_p}(T,\beta^* + 2\varepsilon/3,2\varepsilon/3).
\end{equation}
Inequality \eqref{eq:approach-2} implies that for any $\beta''\in \bR$ and $\varepsilon' \in \bR_{\ge 0}$, $\cC_p\cap \sol_{\cD_p}(T,\beta'',\varepsilon') \subseteq \sol_{\cD}(T,\beta'' + \varepsilon/3,\varepsilon' + \varepsilon/3)$, and thus
\begin{equation}
\label{eq:approach-proof-2}
\cC_p\cap \sol_{\cD_p}(T,\beta^* + 2\varepsilon/3,2\varepsilon/3) \subseteq \sol_{\cD}(T,\beta^* + \varepsilon,\varepsilon).
\end{equation}
Combining \eqref{eq:approach-proof-1} and \eqref{eq:approach-proof-2} completes the proof.
\end{proof}

\section{Proofs for Section~\ref{sec:omni}}
\label{sec:proof-omni}
\subsection {Proof of Theorem~\ref{thm:squared-linear}}
\thmsquaredlinear*
We first prove three helper lemmas/claims below and then prove \Cref{thm:squared-linear}.
\begin{claim}
\label{claim:partial}
For any predictor $p:X\to [0,1]$, any function $f:X\times A\times\{0,1\}\to \bR$ and any $c:X\to A$, we have
\begin{equation}
\label{eq:claim-partial}
\E_{(x,y)\sim \cD}f(x,c(x),y) - \E_{(x,y)\sim \cD_p}f(x,c(x),y)
=
\E_{(x,y)\sim\cD}[(y - p(x))\partial f(x,c(x))],
\end{equation}
where $\partial f(x,a) := f(x,a,1) - f(x,a,0)$ for every $(x,a)\in X\times A$.
\end{claim}
\begin{proof}
The claim is proved by plugging the following equation into the left-hand side of \eqref{eq:claim-partial}.
\[
f(x,c(x),y) = f(x,c(x),0) + y\,\partial f(x,c(x)).
\]
We get
\begin{align*}
  & \E_{(x,y)\sim \cD}f(x,c(x),y) - \E_{(x,y)\sim \cD_p}f(x,c(x),y)\\
= {} &  \E_{(x,y)\sim \cD}[f(x,c(x),0) + y\,\partial f(x,c(x))] - \E_{(x,y)\sim \cD_p}[f(x,c(x),0) + y\,\partial f(x,c(x))].
\end{align*}
The distributions $\cD,\cD_p$ are identical on the $x$ part, therefore $f(x,c(x),0)$ cancels out. The distribution $\cD_p$ is defined such that $y=1$ with probability $p(x)$, which finishes the proof.
\end{proof}

\begin{lemma}
\label{lm:squared-linear-1}
In the setting of \Cref{thm:squared-linear}, for every $c\in \cC$, there exists $c'\in \cC_p(g)$ such that for every convex and special group objective/constraint $f:X\times A\times \{0,1\}\to \bR$, it holds that
\begin{align*}
\E_{(x,y)\sim \cD_p}f(x,c'(x),y) & \le \E_{(x,y)\sim \cD}f(x,c(x),y) + \varepsilon/3.
\end{align*}
\end{lemma}
\begin{proof}
By \Cref{claim:partial},
\begin{equation}
\label{eq:squared-linear-1}
\E_{(x,y)\sim \cD}f(x,c(x),y) - \E_{(x,y)\sim \cD_p}f(x,c(x),y) = \E_{(x,y)\sim \cD}[(y - p(x))\partial f(x,c(x))].
\end{equation}
Since $f$ is a special objective/constraint, there exist $\tau_1,\tau_2:[t]\to [-1,1]$ such that $\partial f(x,c(x)) = \tau_1(g(x)) + \tau_2(g(x))c(x)$. By our assumption that $p\in \grpcal(g,\varepsilon/6)$, we have
\[
\E_{(x,y)\sim \cD}[(y - p(x))\tau_1(g(x))] \ge -\varepsilon/6.
\]
By our assumption that $p\in \grpma(\cC,g,\varepsilon/6)$, we have
\[
\E_{(x,y)\sim \cD}[(y - p(x))\tau_2(g(x))c(x)] \ge -\varepsilon/6.
\]
Combining them, we have
\begin{equation}
\label{eq:squared-linear-2}
\E_{(x,y)\sim \cD}[(y - p(x))\partial f(x,c(x))] = \E_{(x,y)\sim \cD}[(y - p(x))(\tau_1(g(x)) + \tau_2(g(x))c(x))] \ge -\varepsilon/3.
\end{equation}
Finally, define $\tau$ such that $\tau(i,v) = \E[c(x)|g(x) = i,p(x) = v]$ and define $c'(x) = \tau(g(x),p(x))$. It is clear that $c'\in \cC_p(g)$. Moreover, by the convexity of $f$, we have
\[
\E_{(x,y)\sim \cD_p}f(x,c'(x),y) \le \E_{(x,y)\sim \cD_p}f(x,c(x),y).
\]
Combining this with \eqref{eq:squared-linear-1} and \eqref{eq:squared-linear-2} completes the proof.
\end{proof}
\begin{lemma}
\label{lm:squared-linear-2}
In the setting of \Cref{thm:squared-linear}, for every $c\in \cC_p(g)$, for every convex and special group objective/constraint $f:X\times A\times \{0,1\}\to \bR$, it holds that
\begin{align*}
\E_{(x,y)\sim \cD}f(x,c(x),y) & \le \E_{(x,y)\sim \cD_p}f(x,c(x),y) + \varepsilon/3.
\end{align*}
\end{lemma}
\begin{proof}
By \Cref{claim:partial},
\begin{equation}
\label{eq:squared-linear-3}
\E_{(x,y)\sim \cD}f(x,c(x),y) - \E_{(x,y)\sim \cD_p}f(x,c(x),y) = \E_{(x,y)\sim \cD}[(y - p(x))\partial f(x,c(x))].
\end{equation}
Since $f$ is convex and special, there exists $\tau:[t]\times A\to [-2,2]$ such that $\partial f(x,a)=\tau(g(x),a)$.
Since $c\in \cC_p(g)$, there exists $\tau':X\times [0,1]\to A$ such that $c(x)=\tau(g(x),p(x))$. Therefore,
\begin{equation}
\label{eq:squared-linear-4}
\E_{(x,y)\sim \cD}[(y - p(x))\partial f(x,c(x))] = \E_{(x,y)\sim \cD}[(y - p(x))\tau(g(x),\tau'(g(x),p(x)))] \le \varepsilon/3,
\end{equation}
where the last inequality holds by our assumption that $p\in \grpcal(g,\varepsilon/6)$.
Combining \eqref{eq:squared-linear-3} and \eqref{eq:squared-linear-4} completes the proof.
\end{proof}
\begin{proof}[Proof of \Cref{thm:squared-linear}]
The proof is completed by applying \Cref{lm:approach} to the setting of \Cref{thm:squared-linear} and observing that \eqref{eq:approach-1} and \eqref{eq:approach-2} in \Cref{lm:approach} can be established by \Cref{lm:squared-linear-1} and \Cref{lm:squared-linear-2}, respectively.
\end{proof}
\subsection{Proof of Theorem~\ref{thm:convex-linear}}
\thmconvexlinear*
We first prove three helper lemmas below and then prove \Cref{thm:convex-linear}.
\begin{lemma}[\citep{GopalanKRSW22}]
\label{lm:c-hat}
Let $c:X\to \bR$ be a function. Let $g:X\to [t]$ be a group partition function.
Let $f:X\times \bR\times \{0,1\}\to \bR$ be a convex $1$-Lipschitz group objective/constraint (\Cref{def:group-constraint,def:convex-special,def:lipschitz}).
Define
$\tau,\tau':[t]\to \bR$ such that $\tau(i) = \E [y|g(x) = i]$ and $\tau'(i)= \E[c(x)|g(x) = i]$ for every $i\in [t]$. 
Assume that $\sum_{i\in [t]}|\E_{(x,y)\sim\cD}[(y - \tau(i))c(x)\one(g(x) = i)]|\le \varepsilon$. 
We have
\[
\E_{(x,y)\sim \cD}[f(x,\tau'(g(x)),y)] \le \E_{(x,y)\sim \cD}[f(x,c(x),y)] + 2\varepsilon.
\]
\end{lemma}
\Cref{lm:c-hat} is essentially Theorem 19 in \citep{GopalanKRSW22}. The only difference is that in \citep{GopalanKRSW22}, the function $f$ is not allowed to depend on $x$, whereas in \Cref{lm:c-hat}, we allow $f$ to depend on the group index $g(x)$ of $x$. The proof in \citep{GopalanKRSW22} can be used here without any essential change.
\begin{lemma}
\label{lm:convex-linear-1}
In the setting of \Cref{thm:convex-linear}, for every $c\in \cC$, 
there exists $c'\in \cC_p(g)$ such that for every convex $1$-Lipschitz group objective/constraint $f:X\times A\times\{0,1\}\to \bR$ with $1$-bounded difference, it holds that
\[
\E_{(x,y)\sim \cD_p}f_0(x,c'(x),y) \le \E_{(x,y)\sim \cD}f_0(x,c(x),y) + \varepsilon/3.\label{eq:convex-linear-1-1}%
\]
\end{lemma}
\begin{proof}
We fix an arbitrary $c\in \cC$ and define $\tau,\tau':[t]\times [0,1] \to [0,1]$ such that $\tau(i,v) = \E[y|g(x) = i,p(x) = v]$ and $\tau'(i,v) = \E[c(x)|g(x) = i,p(x) = v]$ for every $(i,v)\in [t]\times [0,1]$.

By our assumption that $p\in \grpcal(g,\varepsilon/15)$,
\[
\E_{(x,y)\sim \cD}|p(x) - \tau(g(x),p(x))| \le \varepsilon/15.
\]
By our assumption that $p\in \grpmc(\cC,g,\varepsilon/15)$,
\[
\sum_{i\in [t]}\sum_{v\in \mathsf{range}(p)}|\E_{(x,y)\sim \cD}[(y - p(x))c(x)\one(g(x) = i,p(x) = v)]| \le \varepsilon/15.
\]
Combining the inequalities above,
\[
\sum_{i\in [t]}\sum_{v\in \mathsf{range}(p)}|\E_{(x,y)\sim \cD}[(y - \tau(g(x),p(x)))c(x)\one(g(x) = i,p(x) = v)]|\le 2\varepsilon/15.
\] 
Define $c':X\to A$ such that $c'(x) = \tau'(g(x),p(x))$ for every $x\in X$. Taking the groups in \Cref{lm:c-hat} to be $\{x\in X:g(x) = i,p(x) = v\}$ here for $(i,v)\in [t]\times \mathsf{range}(p)$, we have
\begin{equation}
\label{eq:convex-linear-1-3}
\E_{(x,y)\sim \cD}f(x,c'(x),y) \le \E_{(x,y)\sim \cD}f(x,c(x),y) + 4\varepsilon/15.
\end{equation}
By \Cref{claim:partial},
\begin{equation}
\label{eq:convex-linear-1-4}
\E_{(x,y)\sim \cD}f(x,c'(x),y) - \E_{(x,y)\sim \cD_p}f(x,c'(x),y) = \E_{(x,y)\sim \cD}[(y - p(x))\partial f(x,c'(x))].%
\end{equation}
Since we assume that $f$ is a group objective/constraint and it has $1$-bounded difference, there exists $\tau'':[t]\times A\times \to [-1,1]$ such that $\partial f(x,a) = \tau''(g(x),a)$. By our definition $c'(x) = \tau'(g(x),p(x))$,
\[
\E_{(x,y)\sim \cD}[(y - p(x))\partial f(x,c'(x))] = \E_{(x,y)\sim \cD}[(y - p(x))\tau''(g(x),\tau'(g(x),p(x)))].
\]
By our assumption that $p\in \grpcal(g,\varepsilon/15)$,
\begin{equation}
\label{eq:convex-linear-1-5}
\E_{(x,y)\sim \cD}[(y - p(x))\tau''(g(x),\tau'(g(x),p(x)))] \ge -\varepsilon/15.
\end{equation}
Combining \eqref{eq:convex-linear-1-3}, \eqref{eq:convex-linear-1-4}, and \eqref{eq:convex-linear-1-5} proves \eqref{eq:convex-linear-1-1}. 
\end{proof}
\begin{lemma}
\label{lm:convex-linear-2}
In the setting of \Cref{thm:convex-linear}, for every $c\in \cC_p(g)$, for every convex $1$-Lipschitz group objective/constraint $f:X\times A\times\{0,1\}\to \bR$ with $1$-bounded difference, it holds that
\[
\E_{(x,y)\sim \cD}f(x,c(x),y) \le \E_{(x,y)\sim \cD_p}f(x,c(x),y) + \varepsilon/3.
\]
\end{lemma}
\begin{proof}
The proof is similar to the proof of \Cref{lm:squared-linear-2} and we omit the details. In the proof of \Cref{lm:squared-linear-2}, we use the assumption that $p\in \grpcal(g,\varepsilon/6)$ and the fact that 
there exists $\tau:[t]\times A\to [-2,2]$ such that $\partial f(x,a)=\tau(g(x),a)$.
For our $f$ with $1$-bounded difference, we can similarly take $\tau:[t]\times A\to [-1,1]$ and use our assumption that
$p\in \grpcal(g,\varepsilon/15)\subseteq \grpcal(g,\varepsilon/3)$.
\end{proof}
\begin{proof}[Proof of \Cref{thm:convex-linear}]
The proof is completed by applying \Cref{lm:approach} to the setting of \Cref{thm:convex-linear} and observing that \eqref{eq:approach-1} and \eqref{eq:approach-2} in \Cref{lm:approach} can be established by \Cref{lm:convex-linear-1} and \Cref{lm:convex-linear-2}, respectively.
\end{proof}
\subsection{Proof of Theorem~\ref{thm:general}}
\thmgeneral*
We first prove two helper lemmas below and then prove \Cref{thm:general}.
\begin{lemma}
\label{lm:general-1}
In the setting of \Cref{thm:general}, for every $c\in \cC$, there exists $c'\in \cC_p^{\mathsf{rand}}(g)$ such that for every group objective/constraint $f:X\times A\times \{0,1\}\to \bR$ with $1$-bounded difference, it holds that
\begin{align*}
\E_{(x,y)\sim \cD_p}\E_{a\sim c'(x)}f(x,a,y) & \le \E_{(x,y)\sim \cD}f(x,c(x),y) + \varepsilon/3.
\end{align*}
\end{lemma}
\begin{proof}
By \Cref{claim:partial}, 
\begin{equation}
\label{eq:general-1}
\E_{(x,y)\sim \cD}f(x,c(x),y) - \E_{(x,y)\sim \cD_p}f(x,c(x),y) = \E_{(x,y)\sim \cD_p}[(y - p(x))\partial f(x,c(x))].
\end{equation}
Since we assume that $f$ is a group objective/constraint and it has $1$-bounded difference, there exists $\tau:[t]\times A\to [-1,1]$ such that $\partial f(x,a) = \tau(g(x),a)$. By our assumption that $p\in \grplma(\cC,g,\varepsilon/3)$,
\begin{equation}
\label{eq:general-2}
\E_{(x,y)\sim \cD_p}[(y - p(x))\partial f(x,c(x))] = \E_{(x,y)\sim \cD_p}[(y - p(x))\tau(g(x),c(x)))] \ge -\varepsilon/3.
\end{equation}
Combining \eqref{eq:general-1} and \eqref{eq:general-2}, we have
\begin{equation}
\label{eq:general-3}
\E_{(x,y)\sim \cD_p}f(x,c(x),y) \le \E_{(x,y)\sim \cD}f(x,c(x),y) + \varepsilon/3.
\end{equation}
Now we define $\tau':[t]\times [0,1] \to \Delta_A$ such that $\tau'(i,v)$ is the conditional distribution of $c(x)$ given $g(x) = i$ and $p(x) = v$. We define $c':X\to \Delta_A$ such that $c'(x) = \tau'(g(x),c(x))$. Since $f$ is a group objective/constraint, there exists $\tau'':[t]\times A\times\{-1,1\}\to \bR$ such that $f(x,a,y) = \tau''(g(x),a,y)$. Now we have
\begin{align}
\E_{(x,y)\sim \cD_p}f(x,c(x),y) & = \E[\E[f(x,c(x),y)|g(x),p(x)]]\notag\\
& = \E[\E[\tau''(g(x),c(x),y)|g(x),p(x)]]\notag \\
& = \E_x\left[\E_{a\sim \tau(g(x),p(x)),y\sim \ber(p(x))}[\tau''(g(x),a,y)]\right]\notag \\
& = \E_{(x,y)\sim \cD_p}\E_{a\sim c'(x)}[f(x,a,y)].\label{eq:general-4}
\end{align}
Combining \eqref{eq:general-3} and \eqref{eq:general-4} completes the proof.
\end{proof}
\begin{lemma}
\label{lm:general-2}
In the setting of \Cref{thm:general}, for every $c\in \cC_p^{\mathsf{rand}}(g)$, for every group objective/constraint $f:X\times A\times \{0,1\}\to \bR$ with $1$-bounded difference, it holds that
\begin{align}
\E_{(x,y)\sim \cD}\E_{a\sim c(x)}f(x,a,y) & \le \E_{(x,y)\sim \cD_p}\E_{a\sim c(x)}f(x,a,y) + \varepsilon/3.\label{eq:general-2-2}
\end{align}
\end{lemma}
\begin{proof}
By our assumption $c\in \cC_p^{\mathsf{rand}}(g)$, there exists $\tau:[t]\times[0,1]\to \Delta_A$ such that $c(x) = \tau(g(x),p(x))$ for every $x\in X$. Consider the joint distribution of $(x,a,y)$ where $(x,y)\sim \cD$ and $a\sim c(x)$. This distribution can be equivalently defined as follows. We first construct a function $\tau':[t]\times [0,1]\to A$ at random, where $\tau'(i,v)\in A$ is drawn independently from the distribution $\tau(i,v)\in \Delta_A$ for every $(i,v)\in [t]\times [0,1]$. We then draw $(x,y)\sim \cD$ and choose $c(x)= \tau'(g(x),p(x))$. This equivalent construction also works when we replace $\cD$ with $\cD_p$. Therefore, to prove \eqref{eq:general-2-2}, it suffices to prove that for every $\tau':[t]\times[0,1]\to A$,
\begin{align}
\E_{(x,y)\sim \cD}f(x,\tau'(g(x),p(x)),y) & \le \E_{(x,y)\sim \cD_p}f(x,\tau'(g(x),p(x)),y) + \varepsilon/3.\label{eq:general-2-4}
\end{align}
By \Cref{claim:partial},
\begin{align}
\label{eq:general-2-5}
& \E_{(x,y)\sim \cD}f(x,\tau'(g(x),p(x)),y) - \E_{(x,y)\sim \cD_p}f(x,\tau'(g(x),p(x)),y)\notag \\
= {} & \E_{(x,y)\sim \cD}[(y - p(x))\partial f(x,\tau'(g(x),p(x)))].
\end{align}
Since we assume that $f$ is a group objective/constraint and it has $1$-bounded difference, there exists $\tau'':[t]\times A\to [-1,1]$ such that $f(x,a) = \tau''(g(x),a)$. Therefore,
\begin{align}
& \E_{(x,y)\sim \cD}[(y - p(x))\partial f(x,\tau'(g(x),p(x)))]\notag \\
= {} & \E_{(x,y)\sim \cD}[(y - p(x))\tau''(g(x),\tau'(g(x),p(x)))]\notag \\
\le {} & \varepsilon/3,\label{eq:general-2-6}
\end{align}
where the last inequality follows from our assumption $p\in \grpcal(g,\varepsilon/3)$. Combining \eqref{eq:general-2-5} and \eqref{eq:general-2-6} proves \eqref{eq:general-2-4}.
\end{proof}
\begin{proof}[Proof of \Cref{thm:general}]
The proof is completed by applying \Cref{lm:approach} to the setting of \Cref{thm:general} and observing that \eqref{eq:approach-1} and \eqref{eq:approach-2} in \Cref{lm:approach} can be established by \Cref{lm:general-1} and \Cref{lm:general-2}, respectively.
\end{proof}
\subsection{Variant of Theorem~\ref{thm:squared-linear}}
\begin{theorem}
\label{thm:variant}
Let $\cD$ be a distribution over $X\times \{0,1\}$.
Let $A=[0,1]$ be an action set.
Let $\cT$ be a class of tasks that only have group constraints and group objectives that are all special. Let $\cC$ be a class of functions $c:X\to A$. Let $p$ be a predictor in $\grpma_\cD(\cC,g,\varepsilon/6)\cap \grpcal_\cD(g, \varepsilon/6)$ and define $\cC_p^{\mathsf{rand}}(g)$ as in \Cref{def:c-p}. Then $p$ is a $(\cT,\cC,\cC_p(g),\varepsilon)$-omnipredictor on $\cD$.
\end{theorem}
We first prove two helper lemmas below and then prove \Cref{thm:variant}.
\begin{lemma}
\label{lm:variant-1}
In the setting of \Cref{thm:variant}, for every $c\in \cC$, there exists $c'\in \cC_p^{\mathsf{rand}}(g)$ such that for every special group objective/constraint $f:X\times A\times \{0,1\}\to \bR$, it holds that
\begin{align*}
\E_{(x,y)\sim \cD_p}\E_{a\sim c'(x)}f(x,a,y) & \le \E_{(x,y)\sim \cD}f(x,c(x),y) + \varepsilon/3.
\end{align*}
\end{lemma}
\begin{proof}
Using the same argument as in the proof of \Cref{lm:squared-linear-1}, we can show that
\[
\E_{(x,y)\sim \cD_p}f(x,c(x),y) \le \E_{(x,y)\sim \cD}f(x,c(x),y) + \varepsilon/3.
\]
This is the same as \eqref{eq:general-3} as in the proof of \Cref{lm:general-1}, and the rest of the proof follows the same argument as in the proof of \Cref{lm:general-1}.
\end{proof}
\begin{lemma}
\label{lm:variant-2}
In the setting of \Cref{thm:variant}, for every $c\in \cC_p^{\mathsf{rand}}(g)$, for every special group objective/constraint $f:X\times A\times \{0,1\}\to \bR$, it holds that
\begin{align*}
\E_{(x,y)\sim \cD}\E_{a\sim c(x)}f(x,a,y) & \le \E_{(x,y)\sim \cD_p}\E_{a\sim c(x)}f(x,a,y) + \varepsilon/3.
\end{align*}
\end{lemma}
\begin{proof}
The proof follows the same argument as the proof of \Cref{lm:general-2}. In \Cref{lm:general-2}, we use the assumption that $p\in \grpcal(g,\varepsilon/3)$ and that $f$ has $1$-bounded difference. Here we have the assumption that $p\in \grpcal(g,\varepsilon/6)$, and since we assume $f$ is special and $A = [0,1]$, we know that $f$ has $2$-bounded difference.
\end{proof}
\begin{proof}[Proof of \Cref{thm:variant}]
The proof is completed by applying \Cref{lm:approach} to the setting of \Cref{thm:variant} and observing that \eqref{eq:approach-1} and \eqref{eq:approach-2} in \Cref{lm:approach} can be established by \Cref{lm:variant-1} and \Cref{lm:variant-2}, respectively.
\end{proof}
\section{Lipschitz Combination of Constraints}
\label{sec:lip}
We show that all our omniprediction results in \Cref{sec:omni} can be extended to more general constrained loss minimization tasks where we combine the constraints using a Lipschitz function. Specifically, we consider more general tasks where each task $T$ not only has an objective $f_0:X\times A\times \{0,1\}\to \bR$ and constraints $f_j:X\times A\times \{0,1\}$ for $j\in [m]$, but also has a combining function $\Gamma:\bR^m\to \bR$. The task $T$ corresponds to the following optimization problem:
\begin{align}
\minimize_{c:X\to A} \quad & \E_{(x,y)\sim \cD}f_0(x,c(x),y)\label{eq:lip-optimization} \\
\text{s.t.} \quad & \Gamma\left(\E_{(x,y)\sim \cD}f_1(x,c(x),y), \ldots,\E_{(x,y)\sim \cD}f_m(x,c(x),y)\right) \le 0.\nonumber
\end{align}
The task in \eqref{eq:target-optimization} can be viewed as a special case of \eqref{eq:lip-optimization} where $\Gamma$ is the max function: $\Gamma(r_1,\ldots,r_m) = \max(r_1,\ldots,r_m)$. For a task $T$ in the form of \eqref{eq:lip-optimization}, for $\beta\in \bR$ and $\varepsilon\in \bR_{\ge 0}$, we can again define $\sol_\cD(T,\beta,\varepsilon)$ to be the set of randomized action functions $c:X\to \Delta_A$ satisfying 
\begin{align*}
& \E_{(x,y)\sim \cD}\E_{a\sim c(x)}f_0(x,a,y) \le \beta, \quad\text{and}\\
& \Gamma\left(\E_{(x,y)\sim \cD}\E_{a\sim c(x)}f_1(x,a,y), \ldots,\E_{(x,y)\sim \cD}\E_{a\sim c(x)}f_m(x,a,y)\right) \le \varepsilon.
\end{align*}
Correspondingly, for a class $\cC$ consisting of functions $c:X\to \Delta_A$, we define 
\[
\opt_\cD(T,\cC,\varepsilon) := \inf\{\beta\in \bR:\cC\cap \sol_\cD(T,\beta,\varepsilon)\ne \emptyset\}.
\]
We can then similarly define omnipredictors for these tasks in the same way as in \Cref{def:omni}.

Here we focus on obtaining omnipredictors for tasks with Lipschitz combining functions $\Gamma$. We say $\Gamma$ is $\kappa$-Lipschitz (in the $\ell_\infty$ norm) if $|\Gamma(r_1,\ldots,r_m) - \Gamma(r_1',\ldots,r_m')| \le \kappa \max_{i\in [m]} |r_i - r_i'|$. For tasks with $1$-Lipschitz combining functions, we have the following analogue of \Cref{lm:approach}:
\begin{lemma}
\label{lm:lip}
Let $\cT$ be a class of constrained loss minimization tasks each having a $1$-Lipschitz combining function.
Let $\cC$ and $\cC_p$ be classes of action functions $f:X\to \Delta_A$
as in \Cref{def:omni}. If a predictor $p$ satisfies the following two properties for every $T\in \cT$, then $p$ is a $(\cT,\cC,\cC_p,\varepsilon)$-omnipredictor:
\begin{enumerate}
\item Let $f_0$ be the loss function of $T$ and $(f_j)_{j\in J}$ be the constraints of $T$. For every $c\in \cC$, there exists $c'\in \cC_p$ such that
\begin{align}
\E_{(x,y)\sim \cD_p}\E_{a\sim c'(x)}f_0(x,a,y)
\le {} & \E_{(x,y)\sim \cD}\E_{a\sim c(x)}f_0(x,a,y) + \varepsilon/3, \quad \text{and}\label{eq:lip-1}\\
\Big|\E_{(x,y)\sim \cD_p}\E_{a\sim c'(x)}f_j(x,a,y)
- {} & \E_{(x,y)\sim \cD}\E_{a\sim c(x)}f_j(x,a,y)\Big| \le \varepsilon/3 \quad \text{for every }j\in J.\label{eq:lip-2}
\end{align}
\item For every $c\in \cC_p$,
\begin{align}
\E_{(x,y)\sim \cD}\E_{a\sim c(x)}f_0(x,a,y) \le {} & \E_{(x,y)\sim \cD_p}\E_{a\sim c(x)}f_0(x,a,y) + \varepsilon/3,\quad \text{and}\label{eq:lip-3}\\
\Big|\E_{(x,y)\sim \cD}\E_{a\sim c(x)}f_j(x,a,y)
- {} & \E_{(x,y)\sim \cD_p}\E_{a\sim c(x)}f_j(x,a,y)\Big| \le \varepsilon/3 \quad\text{for every }j\in J.\label{eq:lip-4}
\end{align}
\end{enumerate}
\end{lemma}
\Cref{lm:lip} can be proved similarly to \Cref{lm:approach} using the observation that \eqref{eq:lip-2} implies the following by the $1$-Lipschitz assumption on $\Gamma$ and an analogous observation for \eqref{eq:lip-4}:
\begin{align*}
& \Gamma\left(\E_{(x,y)\sim \cD_p}\E_{a\sim c'(x)}f_1(x,a,y), \ldots,\E_{(x,y)\sim \cD_p}\E_{a\sim c'(x)}f_m(x,a,y)\right)\\
\le {} & \Gamma\left(\E_{(x,y)\sim \cD}\E_{a\sim c(x)}f_1(x,a,y), \ldots,\E_{(x,y)\sim \cD}\E_{a\sim c(x)}f_m(x,a,y)\right) + \varepsilon/3.
\end{align*}
We thus omit the proof of \Cref{lm:lip}. The only difference between \Cref{lm:lip} and \Cref{lm:approach} in the requirements needed for $p$ to be an omnipredictor is the additional absolute values in \eqref{eq:lip-2} and \eqref{eq:lip-4}. As all our proofs in \Cref{sec:omni} are through \Cref{lm:approach}, they can be adapted to tasks with constraints combined by a Lipschitz function using \Cref{lm:lip}. The absolute values in \eqref{eq:lip-2} and \eqref{eq:lip-4} only require us to make sure that for every constraint function $f$, both $f$ and $-f$ satisfy the assumptions needed for our theorems in \Cref{sec:omni} (e.g. we need to replace ``convex'' by ``affine''). Note that all linear constraints $f$ defined in \eqref{eq:linear-constraint} satisfy that both $f$ and $-f$ are convex and special. Ideas in this section can be applied to tasks where the objective function is also a Lipschitz combination:
\begin{align*}
\minimize_{c:X\to A} \quad & \Gamma'\left(\E_{(x,y)\sim \cD}f_1'(x,c(x),y),\ldots,\E_{(x,y)\sim \cD}f_{m'}'(x,c(x),y)\right)\nonumber \\
\text{s.t.} \quad & \Gamma\left(\E_{(x,y)\sim \cD}f_1(x,c(x),y), \ldots,\E_{(x,y)\sim \cD}f_m(x,c(x),y)\right) \le 0.
\end{align*}
\section{Rank-preserving Transformations of Omnipredictors}\label{sec:rank-p}
Loss functions are meant to represent the cost of an action given the true value $y$. In particular, they represent the cost of a bad prediction. In a loans setting, it might be the money that the bank losses if the loans is not returned. 

In this work we show that certain predictors are omnipredictors even under constraints. In this section, we show how to achieve an additional property we call rank-preserving. 
\begin{definition}
A transformation $\tau:[0,1]\times[t]\rightarrow[0,1]$ is rank-preserving within groups, if for every $i\in[t]$, the function $\tau_i:[0,1]\rightarrow[0,1]$ defined by $\tau_i(v) = \tau(i,v)$ is a monotonically increasing function, i.e. for every $v>v'\in [0,1]$, $\tau(i,v)\geq \tau(i,v')$.
\end{definition}
This property is desired when we want to assign high value of the action set $A\subset [0,1]$ to individuals with high probability for a positive outcome. For example, if $p(x)$ is the probability of an individual $x$ to return a loan, the bank should give higher loans to individuals with higher value of $p(x)$.

We can guarantee this property only for optimization tasks in which the loss function is also rank-preserving. 
\begin{definition}
A loss function $f_0:X\times[0,1]\times\{0,1\}\rightarrow[0,1]$ is rank-preserving within groups, if there exists a function $f:[t]\times[0,1]\times\{0,1\}$ such that for all $x\in X,a\in [0,1],y\in\set{0,1}$, we have $f_0(x,a,y) = f(g(x),a,y)$ and the function $f$  satisfies for all $i\in[t]$ and $v>v'\in[0,1]$,
\begin{align*}
     &f(i,v,1)\leq f(i,v',1)\\
    &f(i,v,0)\geq f(i,v',0).
\end{align*}
\end{definition}
Rank preserving a desired property when the loss function represents the distance between the taken action and the outcome. In particular, the $\ell_1$ loss and squared loss satisfy it, as well as every loss function of form $f(x,a,y) =\text{dist}(a,y)$, when $\text{dist}$ is a distance function.

We show a post-processing algorithm that takes a transformation $\tau$ and transforms it into a rank-preserving transformation while preserving the constraints and without increasing the loss. In order to do so, we need that the omnipredictor we apply the transformation on to also be monotone. 
\begin{definition}
A predictor $p:\cX\rightarrow[0,1]$ with a discrete range $V$ is monotone if for every $v>v'\in V$, $\E_{(x,y)\sim \cD}[y|p(x)=v] > \E_{(x,y)\sim \cD}[y|p(x)=v]$.
\end{definition}
This is a natural requirement, and we show that calibrated predictor with a discrete range can be modified to one that is monotone with high probability, by merging small level sets and level sets that are close together. This claim only holds for functions $w$ with bounded range, although the rest of the section holds more generally. We remark that as long as the hypothesis class $H$ contains bounded functions $h:X\rightarrow[0,1]$, then the claim below holds for all classes $W$ defining group or level-set calibration on \Cref{sec:group-ma-mc}. In case of group multi-accuracy or calibration with negative value of $\tau$, the claim below should be run on each part $\set{x|g(x)=i}$ separately.

\begin{claim}
    Let $V\subset[0,1]$ be a discrete set, and let $W$ be a class of functions $w:X\times[0,1]\rightarrow [0,1]$ containing a function $f_v(x,v') = \one(v=v')$ for all $v\in V$. 
    Let $p:X\rightarrow[0,1]$ be a predictor with a discrete range $V$ such that $p\in\genmc_\cD(W,\varepsilon)$. Then there is an algorithm running in time $O(\abs{V}^3\frac{1}{\varepsilon^2 \delta})$, uses $O(\abs{V}^3\frac{1}{\varepsilon^2\delta})$ samples, that with probability $1-\delta$ outputs a  monotone predictor $p'\in \genmc_\cD(W,6\varepsilon)$.
\end{claim}
\begin{proof}
    We describe a simple  algorithm for merging the levels of $p$ that are too close to each other or too small. We start by looking at the partitionof $X$ defined by $p$, then merge parts that are too small or too close to each other. Let $P = P_1,\ldots P_{\abs{V}}$ be the partition of $x$ defined by $p$.

    The algorithm sample $S$ of size $O(\abs{V}^3\frac{1}{\varepsilon^2\delta})$ of $(x,y)\sim\cD$, and do:
    \begin{enumerate}
        \item While there exists a part $P_i$ such that $\Pr_{(x,y)\in S}[x\in P_i]< \frac{2\varepsilon}{\abs{\cR}}$, merge $P_i$ with its neighbor. \label{step:small-merge}
        \item While there are $P_i,P_j\in P$ such that 
        \[ \abs{\E_{(x,y)\in S}[y|x\in P_i]-\E_{(x,y)\in S}[y|x\in P_j]} <\frac{2\varepsilon}{\abs{V}},  \]
        merge $P_i,P_j$.\label{step:merge}
        \item Set $p':X\rightarrow[0,1]$ by choosing for every part $x\in P_i$ the value $\E_{(x',y')\in S}[y'|x'\in P_i]$.
    \end{enumerate}

    From \Cref{claim:est-2}, by taking $O(\abs{V}^3\frac{1}{\varepsilon^2\delta})$, with probability $1-\delta/2$ the algorithm approximates $\Pr_{(x,y)\sim\cD}[p(x)=v]$  up to an error of $\frac{\varepsilon}{\abs{V}}$. After the first step of the algorithm, each $P_i$ has size at least $\frac{\varepsilon}{\abs{V}}$. Therefore, from \Cref{claim:est-2} the algorithm approximates $\E_{(x,y)\in S}[y|x\in P_i]$ up to an additive error of $\frac{\varepsilon}{\abs{V}}$ with probability $1-\delta/2$ for all parts. Assuming all approximations are correct, the predictor $p'$ is monotone. Therefore, $p'$ is monotone with probability at least $1-\delta$.

    To prove the generalized calibration, we first use the function $f_v\in W$ and get that for every $v\in V$,
    \begin{align}
        \abs{\E_{(x,y)\sim\cD}[(y-v)\one(p(x)=v)]}\leq\varepsilon, \label{eq:p-calib}
    \end{align}

    Assume that the algorithm skips \Cref{step:merge}, and only preforms merging for small sets and asigns new values. Let $p''$ be this predictor. Then for $p''$ we have, 
    \begin{align}
        & \abs{\E_{(x,y)\sim\cD}[(y-p''(x))w(x,v)] }\nonumber \\ \leq {} & 
        \abs{\E_{(x,y)\sim\cD}[(y-p''(x))w(x,v)] + \E_{(x,y)\sim\cD}[(p''(x) - p(x))w(x,v)]}\nonumber \\
        \leq {} & \varepsilon + 
        \abs{\E_{(x,y)\sim\cD}[(p''(x)-p(x))w(x,v)] }\nonumber\\
        \leq {} & \varepsilon + \abs{ \Pr_{(x,y)\sim\cD}[p(x) \text{ in small }P_i] + \sum_{\text{large }P_i}\E_{(x,y)\sim\cD}[(p''(x)-p(x))w(x,v)\one(x\in P_i)] }\nonumber\\
        \leq {} & 3\varepsilon + \abs{\sum_{\text{large }P_i}\E_{(x,y)\sim\cD}[(p''(x)-p(x))\one(x\in P_i)] }\nonumber\\
        \leq {} & 3\varepsilon + \abs{\E_{(x,y)\sim\cD}[(y-v)\one(p(x)=v)]} + \sum_{\text{large }P_i}\abs{\E_{(x,y)\sim\cD}[(y-v)\one(p''(x)=v)]}.\label{eq:bound-prime}
    \end{align}
    Where large $P_i$'s are those that the algorithm does not merge in \Cref{step:small-merge}.
    From \cref{eq:p-calib}, the first expectation is bounded by $\varepsilon$. From the paragraph above, with probability at least $1-\delta/2$ the we have $\abs{\E_{(x,y)\in S}[y|x\in P_i] - \E_{(x,y)\sim \cD}[y|x\in P_i]}\leq\varepsilon/\abs{V}$ for all large partitions $P_i$. Together we get $\abs{\E_{(x,y)\sim\cD}[(y-p''(x))w(x,v)] }\leq5\varepsilon$.

    Our monotone predictor $p'$ has an extra step in \Cref{step:merge}, in which the algorithm merges parts $P_i,P_j$. The algorithm only merges parts in which the expected value of $y$, $\E[y|x\in P_i]$ is within distance $\frac{\varepsilon}{V}$. Therefore, even if we preform $\abs{V}$ merges,  we have that 
    \[ \E_{(x,y)\sim\cD}[\abs{p'(x) - p''(x)}]\leq \varepsilon. \]
    Substituting $p'(x)$ instead of $p''(x)$ on equation \cref{eq:bound-prime} can only increase the expected value by $\varepsilon$.
\end{proof}

We show a post-processing algorithm taking a transformation $\tau$, and transforming it into a rank-preserving transformation $\tau'$ without increasing the loss and violating the constraints. We start with the simpler case, in which we can have a deterministic transformation. 
\begin{lemma}[\Cref{lem:det-rank-p} restated]
    Let $A=[0,1]$ be the action set, $\cT$ be class of tasks with linear constraints. Assume that for every task $T\in\cT$ the objective function is rank-preserving and convex, and that for every group $i\in[t]$, there is only a single constraint $f_j$ in which $\tau_1(i),\tau_2(i),\tau_3(i)\neq 0$, and $\text{sign}(\tau_2(i))=\text{sign}(\tau_3(i))$. Then for every monotone omnipredictor $p$ we have 
\[    
    \opt_\mathcal{D}(T,\text{rank-preserving }\cC_p(g),\varepsilon)\leq\opt_\mathcal{D}(T,\cC_p(g),\varepsilon).
\]
    Furthermore, given any deterministic $c\in \cC_p(g)$, such that $c(x) = \tau(g(x),p(x))$ for a transformation $\tau:[t]\times V\rightarrow A$, there exists an algorithm running in time polynomial in $t,\abs{V},\varepsilon$ and outputting a transformation $\tilde\tau:[t]\times V\rightarrow A$ that is rank-preserving, and $c'(x) = \tilde\tau(g(x),p(x))$ has the same objective value as $c$ up to a factor of $\varepsilon$ with high probability.
\end{lemma}
We remark that the requirement $\text{sign}(\tau_2(i))=\text{sign}(\tau_3(i))$ is necessary. When picking a constraint in which the expected value and outcome-aware expected value are at odds, it might be that the only solution is not rank-preserving. This highlights the importance of picking appropriate loss functions and constraints if we want to achieve fair outcome.

\begin{proof}
    We prove the claim by an iterative process, taking $\tau$ that is not rank-preserving on some inputs and correcting it.

    Suppose $\tau$ is not rank-preserving, and there exists $i\in[t],v>v'\in V$ such that $\tau(i,v)<\tau(i,v')$. We show how to correct $\tau$ on the values $v,v'$, and let $\tau'$ be the transformation after the correction. Denote 
    \begin{align}
        \beta =& \Pr_{(x,y)\sim\cD}[p(x)=v|p(x)\in\set{v,v'}]\\
        q_v =& \E_{(x,y)\sim\cD}[y|g(x)i,p(x)=v]\\
        q_{v'} =& \E_{(x,y)\sim\cD}[y|g(x)i,p(x)=v']
    \end{align}
    
    Let $f_j$ be the single constraint which affect $g(i)$, and write $f_j(x,a,y) = \alpha_1  + \alpha_2 a + \alpha_3 a y$ for $\alpha_1,\alpha_2,\alpha_3\in \mathbb{R}$.  Then the value of the constraint on $x$'s such that $g(x)=i, p(x)\in\set{v,v'}$ is
    \begin{align*} & \E_{(x,y)\sim D}[f_j(x,a,y)|g(x)=i,p(x)\in\set{v,v'}]\\ = {} & \alpha_1 + \alpha_2 (\beta\tau(i,v)+(1-\beta)\tau(i,v')) + \alpha_3(\beta q_v \tau(i,v) + (1-\beta) q_{v'}\tau(i,v')). \end{align*}

    We want to switch the value of $\tau$ on $v,v'$, while still satisfying the constraint. Simple switch does not work, as the constraint can be violated. Instead, we switch one of them, and set the other to a different value in order to satisfy the constraint. Assume without loss of generality that  $\text{sign}(\alpha_2(1-2\beta) + \alpha_3(1-\beta) q_{v'} - \beta q_{v} = \text{sign}(\alpha_2)$. In this case we set $\tau'(i,v)=\tau(i,v')$, and set $\tau'(i,v')$ to equal $z$ defined by
    \begin{align}
        z =& \frac{\E_{(x,y)\sim D}[f_j(x,a,y)|g(x)=i,p(x)\in\set{v,v'}] - \alpha_2\beta\tau(i,v') - \alpha_3\beta q_v\tau(i,v') }{\alpha_2(1-\beta) + \alpha_3(1-\beta) q_{v'}}\\
        =& \frac{\alpha_2 \beta+\alpha_3 \beta q_v}{\alpha_2(1-\beta) + \alpha_3(1-\beta) q_{v'}}\tau(i,v) + \frac{\alpha_2(1- 2\beta)+\alpha_3 ((1-\beta) q_{v'}-\beta q_{v})}{\alpha_2(1-\beta) + \alpha_3(1-\beta) q_{v'}}\tau(i,v').
    \end{align}
    Since $\text{sign}(\alpha_2) = \text{sign}(\alpha_3)$, and our assumption,  we can see that $z = \gamma \tau(i,v) + (1-\gamma)\tau(i,v')$ for some value $\gamma\in[0,1]$. This implies that $z\in [0,1]$, as required. Using the fact that $q_v > q_{v'}$ we can further bound $\gamma$ by $\frac{\beta}{1-\beta}\leq\gamma\leq \frac{\beta q_v}{(1-\beta)q_{v'}}$.

    The value $z$ is chosen such that $\E_{(x,y)\sim D}[f_j(x,a,y)|g(x)=i,p(x)\in\set{v,v'}]$ is unmodified by the change, meaning the the constraint $f_j$ is kept.
    
    We are left with showing that the objective is not reduced by the correction. Since the objective value is an expectation and therefore additive, it is enough to analyze the value for $x$ such that $g(x)=i,p(x)\in\set{v,v'}$.

    The original objective value for these $x$'s was
    \begin{align*}
        & \E_{(x,y)\sim D}[f_0(x,a,y)|g(x)=i,p(x)\in\set{v,v'}]\\
         = {} & \beta q_v f_0(i,\tau(i,v),1) + \beta (1-q_v) f_0(i,\tau(i,v),0) \\&+(1-\beta) q_{v'} f_0(i,\tau(i,v'),1) + (1-\beta) (1-q_{v'}) f_0(i,\tau(i,v'),0).  
    \end{align*}
    The loss after the modification is
    \begin{align*}
        \ell' = {} & \beta q_v f_0(i,\tau(i,v'),1) + \beta (1-q_v) f_0(i,\tau(i,v'),0)\\
        & +(1-\beta) q_{v'} f_0(i,z,1) + (1-\beta) (1-q_{v'}) f_0(i,z,0) \\
        \leq {} & \beta q_v f_0(i,\tau(i,v'),1) + \beta (1-q_v) f_0(i,\tau(i,v'),0)\\ & + (1-\beta) q_{v'}(\gamma f_0(i,\tau(i,v),1) +(1-\gamma)f_0(i,\tau(i,v'),1)) \\&+ (1-\beta) (1-q_{v'})(\gamma f_0(i,\tau(i,v),0)+(1-\gamma)f_0(i,\tau(i,v'),0))
        \\= {} & f_0(i,\tau(i,v'),1)(\beta q_v+ (1-\beta)q_{v'}(1-\gamma))\\& + f_0(i,\tau(i,v'),0)(\beta (1-q_v) +(1-\beta) (1-q_{v'})(1-\gamma))\\&+
        f_0(i,\tau(i,v),1)(1-\beta) q_{v'}\gamma +  f_0(i,\tau(i,v),0)(1-\beta) (1-q_{v'})\gamma
    \end{align*}
    We subtract the two values,
    \begin{align*}
       &\E_{(x,y)\sim D}[f_0(x,a,y)|g(x)=i,p(x)\in\set{v,v'}]-\ell'\\
       \geq {} & 
       (f_0(i,\tau(i,v),1)-f_0(i,\tau(i,v'),1) )(\beta q_v-\gamma (1-\beta)q_{v'}) \\ &+ (f_0(i,\tau(i,v'),0)-f_0(i,\tau(i,v),0))(\gamma (1-\beta)(1-q_{v'}) - \beta(1-q_v))
    \end{align*}
    From our assumption that $f_0$ is rank-preserving and $\tau(i,v')\geq \tau(i,v)$, we have that $f_0(i,\tau(i,v),1)-f_0(i,\tau(i,v'),1)  \geq 0$, and similarly that $f_0(i,\tau(i,v'),0)-f_0(i,\tau(i,v),0)\geq 0$. From the bounds on $\gamma$ we have that
    \[ \beta q_v-\gamma (1-\beta)q_{v'}\geq \beta q_v- (1-\beta)q_{v'}\frac{\beta q_v}{(1-\beta)q_{v'}}\geq 0, \]
    where the last inequality is because $q_v > q_{v'}$.
    Similarly we have 
    \[\gamma (1-\beta)(1-q_{v'}) - \beta(1-q_v) \geq \frac{\beta}{1-\beta} (1-\beta)(1-q_{v'}) - \beta(1-q_v)\geq 0.\]
    Together we get that $\E_{(x,y)\sim D}[f_0(x,a,y)|g(x)=i,p(x)\in\set{v,v'}]-\ell'\geq 0$, therefore correcting $\tau$ does not increase the objective. 

    After preforming such replacements for every pair $v,v'$ such that $\tau$ is not rank-preserving on, the resulting transformation $\tilde{\tau}$ is rank-preserving.

    The correction described above is existential, because it assumes the exact value of $\beta,q_v,q_{v'}$. In order to implement such algorithm in practice, we should just approximate $\beta,q_v,q_{v'}$, and update $\tau$ based on our approximation. Because this is an approximation, in practive the objective can be reduced by a value proportional to the accuracy parameter of our approximation. The running time of such algorithm is polynomial in $\abs{V},t,\delta$ when $\delta$ is the accuracy parameter.
\end{proof}

The previous theorem modified a transformation $\tau$ into a rank-preserving one by ``correcting'' its values for every violation. Allowing the correction to be randomized, the theorem holds for a larger collection of constraints. In order to do so, we first define rank-preserving for a randomized transformation.
\begin{definition}
    A randomized transformation $\tau:[0,1]\times[t]\Delta_A$ for $A=[0,1]$ is rank-preserving within groups, if for every $i\in[t]$, $v>v'\in V$ and $\gamma\in[0,1]$
    \[ \Pr[\tau(i,v)\geq\gamma]\geq \Pr[\tau(i,v')\geq\gamma]. \]
\end{definition}
\begin{lemma}[\Cref{lem:rand-rank-p} restated]
    Let $A\subseteq[0,1]$ be a discrete action set, $\cT$ be class of tasks with constraints that are independent of the outcome.  Then for a monotone omnipredictor $p$ we have $ \opt_\mathcal{D}(T,\text{rank-preserving }\cC_p^{\mathsf{rand}}(g),\varepsilon)=\opt_\mathcal{D}(T,\cC_p^{\mathsf{rand}}(g),\varepsilon)$.
    
    Furthermore, given any $c\in \cC_p^{\mathsf{rand}}(g)$, such that $c(x) = \tau(g(x),p(x))$ for a transformation $\tau:[t]\times V\rightarrow A$, there exists an algorithm running in time polynomial in $t,\abs{V},\varepsilon$ and outputting a randomized transformation $\tilde\tau:[t]\times V\rightarrow A$ that is rank-preserving, and $c'(x) = \tilde\tau(g(x),p(x))$ has the same objective value as $c$ up to a factor of $\varepsilon$ with high probability.
\end{lemma}
\begin{proof}
    The proof follows the same structure of the previous proof. Let $\tau$ be a randomized transformation, and assume that there exists $v>v'$ such that $\tau$ is not rank-preserving on $v,v'$. We describe a single step in an iterative process, transforming $\tau$ into $\tau'$.

    Intuitively, we take the histogram of the values of $\tau$ on the input set $\{x\in X|g(x)=i,p(x)\in \{v,v'\}\}$, and assign $v'$ the lower values in the histogram and $v$ the upper ones.

    We define
    \begin{align}
        \beta =& \Pr_{(x,y)\in\cD}[p(x)=v | g(x)=i, p(x)\in \{v,v'\}]\\
        \beta_a=&  \beta \Pr[\tau(i,v)=a] + (1-\beta)\Pr[\tau(i,v')=a], \quad\forall a\in A
    \end{align}
    when the probability in the second definition is over the internal randomness of $\tau$. 
    For every $a\in A$, we define the function $u:A\rightarrow[0,1]$ indicating how much of $\beta_a$ is coming from $\tau(i,v)$. That is, for all $a\in A$ if $\beta\neq 0$ we have
    \[ u(a) = \beta\Pr[\tau(i,v)=a]/\beta_a. \]
    When $\beta_a=0, u(a)$ can take any value in $[0,1]$. Notice that by definition, $\Pr[\tau(i,v')=a] = \beta_a(1-u(a))/(1-\beta)$.
    
    We define $\tau'$ by creating an analog function $u':A\rightarrow[0,1]$, when $u'$ indicates if a certain outcome $a\in A$ is in the upper part of the histogram (and should be assigned to $\tau'(i,v)$) or lower part (and should be assigned to $\tau'(i,v')$). Fractional values $u'(a)$ imply that $a$ is in the middle of the histogram, i.e. assigned to both.
For every $a\in A$ let
\begin{align}
    &u'(a) = \begin{cases}
    1 \quad &\text{if } \sum_{a'\geq a}\beta_{a'}\leq\beta\\
    0 \quad &\text{if } \sum_{a'\leq a}\beta_{a'}\leq 1-\beta\\
    \frac{1}{\beta_a}\paren{\beta - \sum_{a'> a}\beta_{a'}} \quad &\text{otherwise}.
    \end{cases}
\end{align}

We are now ready to define $\tau'$ to equal $\tau$ on all except $(i,v),(i,v')$, in which we have:
\begin{align}
    &\forall a\in A \quad \Pr[\tau'(i,v)=a] = \frac{\beta_a u'(a)}{\beta} \\
    &\forall a\in A \quad \Pr[\tau'(i,v)=a] = \frac{\beta_a}{1-\beta}(1-u'(a)).
\end{align}
Notice that $\tau'$ is rank-preserving on inputs $(i,v),(i,v')$ by definition.

We next show that $\tau'$ satisfies all of the constraints in the same way was $\tau$. Let $f_j(i,a,y) = f(i,a)$ be any constraint that is not a function of $y$. Then we have
\begin{align*}
    \E_{(x,y)\sim\cD}[f(i,\tau(i,p(x))) ] = \sum_{a\in A}\Pr_{(x,y)\sim\cD}[\tau(i,p(x))=a]f(i,a).
\end{align*}
The transformations $\tau,\tau'$ only differ on inputs $(i,v),(i,v')$, so it is enough to analyze the difference on these inputs. 
For every $a\in A$,
\[ \Pr_{(x,y)\sim\cD}[\tau(i,p(x))=a|g(x)=i,p(x)\in\set{v,v'}] = \beta \Pr[\tau(i,v)=a] + (1-\beta)\Pr[\tau(i,v')=a] = \beta_a.  \]
For the new transformation,
\begin{align*}
    \Pr_{(x,y)\sim\cD}[\tau'(i,p(x))=a|g(x)=i,p(x)\in\set{v,v'}] &= \beta \Pr[\tau'(i,v)=a] + (1-\beta)\Pr[\tau'(i,v')=a] \\&= \beta \frac{\beta_a u'(a)}{\beta}+ (1-\beta)\frac{\beta_a}{1-\beta}(1-u'(a)) = \beta_a. 
\end{align*}
Therefore, we get that $\E_{(x,y)\sim\cD}[f(i,\tau(i,p(x))) ]=\E_{(x,y)\sim\cD}[f(i,\tau'(i,p(x))) ]$.

We are left with proving that this correction does not increase the loss. We define $q_v,q_{v'}$ as in the previous proof.
\begin{align}
          q_v =& \E_{(x,y)\sim\cD}[y|g(x)i,p(x)=v]\\
        q_{v'} =& \E_{(x,y)\sim\cD}[y|g(x)i,p(x)=v'].
\end{align}
The expected loss of $\tau$ on the relevant inputs:
\begin{align*}
   & \E_{(x,y)\in\cD}[f_0(i,\tau(i,x),y) | g(x)=i, p(x)\in \{v,v'\}]\\ = {} & \beta \sum_{a\in A}\Pr[\tau(i,v)=a]\paren{q_vf_0(i,a,1) + (1-q_v)f_0(i,a,0)} \\&+(1-\beta)\sum_{a\in A}\Pr[\tau(i,v')=a]\paren{q_{v'}f_0(i,a,1) + (1-q_{v'})f_0(i,a,0)}\\= {} &
    \sum_{a\in A}f_0(i,a,1)\beta_a \paren{u(a) q_v  + (1-u(a))q_{v'}} \\&+ \sum_{a\in A}f_0(i,a,0)\beta_a\paren{u(a)(1- q_v) + (1-u(a))(1-q_{v'})}.
\end{align*}
By definition, the loss of $\tau'$ is exactly the same only with $u'$ instead of $u$.

Comparing the two losses we get:
\begin{align}\label{eq:loss-diff}
    &\E_{(x,y)\in\cD}[f_0(i,\tau(i,x),y) | g(x)=i, p(x)\in \{v,v'\}] - \E_{(x,y)\in\cD}[f_0(i,\tau'(i,x),y) | g(x)=i, p(x)\in \{v,v'\}] \\&=
    \sum_{a\in A}\beta_a(f_0(i,a,1)- f_0(i,a,0))(u(a)-u'(a))(q_v-q_{v'}).
\end{align}
Denote $\gamma_a = (u(a)-u'(a))(q_v-q_{v'})$.
From our assumption, $q_v\geq q_{v'}$. From the definition of $u'(a)$, for every $a\in A$ we have 
 \[  \sum_{a'\geq a\in A}u'(a)\geq \sum_{a'\geq a\in A}u(a). \]
Since $\sum_{a\in A}u(a) = \sum_{a\in A}u'(a)$, we have that $\sum_a{\gamma_a}=0$, and that there exists $\tilde{a}$ such that $\gamma_a\leq 0$ for all $a>\tilde{a}$, and $\gamma_a\geq 0$ for all $a\geq\tilde{a}$.  Since the function $f_0$ is rank preserving, we have that for every $a>a'$,
\[ f_0(i,a,1)- f_0(i,a,0) \leq f_0(i,a',1)- f_0(i,a',0). \]
Therefore, 
\[ \sum_{a,\gamma_a \leq 0}\gamma_a(f_0(i,a,1)- f_0(i,a,0)) \leq \sum_{a,\gamma_a \geq 0}\gamma_a(f_0(i,a,1)- f_0(i,a,0)). \]
Which implies that $\sum_{a}\gamma_a(f_0(i,a,1)- f_0(i,a,0))\geq 0$ and the loss of $\tau'$ is at most the loss of $\tau$.

The final transformation $\tilde{\tau}$ is created by repeatedly applying the above step until $\tilde{\tau}$ is rank-preserving. The process ends after $\abs{V}^2$ such switching steps.

When preforming the algorithm in practive we do know know $u,\beta,q_v,q_{v'}$ exactly and need to approximate them at every step. This adds an error to the algorithm.
\end{proof}

\section{Algorithms for Multiaccuracy and Multicalibration}
\label{sec:algo-ma-mc}
The computational and sample complexity of learning
a multiaccuracy/multicalibrated predictor w.r.t.\ a function class $\cC$ using i.i.d.\ data points from the true distribution $\cD$ depends on the complexity and structure of the class $\cC$.
In \citep{hebert2018multicalibration}, the authors show that 
the task can be \emph{efficiently} reduced to \emph{weak agnostic learning} for $\cC$ \citep{MR2582918,DBLP:conf/innovations/Feldman10}.
This implies that the sample and computational complexity of learning a multicalibrated predictor cannot be much larger than weak agnostic learning. 
\citet{hu2022metric} concretely characterize the sample complexity of learning a multiaccurate/multicalibrated predictor in terms of the \emph{fat-shattering dimension} of $\cC$ \citep{MR1279411}, and they also study the sample complexity of multiaccuracy/multicalibration with additional realizability assumptions about $\cD$, which is a setting further explored by \citet{hu2022comparative} (results in our paper do not require any assumption on $\cD$).
\citet{gopalan2022loss} propose and implement algorithms for calibrated multiaccuracy and demonstrate their efficiency compared to achieving multicalibration.
Many of our results in this paper require group multiaccuracy/multicalibration, and 
such a predictor can be obtained by first learning a multiaccurate/multicalibrated predictor w.r.t.\ $\cC$ on each group and then combining. 
Some of our results in this paper require group level-set multiaccuracy. This can be equivalently viewed as multiaccuracy w.r.t.\ a larger class $\cC'$ of binary functions $c':X\to \{-1,1\}$ such that there exist $c\in \cC$ and $\tau:[t]\times A\to \{-1,1\}$ satisfying $c'(x) = \tau(g(x),c(x))$ for every $x\in X$. The complexity of $\cC'$ depends on the complexity of $\cC$ and the group partition $g$.
\section{Optimization Algorithms on the Simulated Distribution}
\label{sec:algo-simulated}
An omnipredictor $p$, as in \Cref{def:omni}, allows us to solve downstream tasks $T\in \cT$ on the true distribution $\cD$ by solving the task on the simulated distribution $\cD_p$. In this section, we show very efficient algorithms for solving the task on the simulated distribution for all the settings we consider in \Cref{sec:omni}.

Specifically, in \Cref{def:omni}, we define $\beta:= \opt_{\cD_p}(T,\cC_p,\varepsilon/3)\in \bR$. Suppose the objective of $T$ is $f_0:X\times A\times\{0,1\}\to \bR$ and the constraints of $T$ are $f_j:X\times A\times\{0,1\}\to \bR$ for every $j\in J$. The task of finding a solution in $\cC_p\cap \sol_{\cD_p}(T,\beta + \varepsilon/3,2\varepsilon/3)$ is to solve the following optimization problem approximately:
\begin{align}\label{eq:optimize}
\minimize_{c\in \cC_p} \quad & \E_{(x,y)\sim \cD_p}\E_{a\sim c(x)}f_0(x,a,y)\\
\text{s.t.} \quad & \E_{(x,y)\sim \cD_p}\E_{a\sim c(x)}f_j(x,a,y)\le 0 \quad \text{for every }j\in J.\nonumber
\end{align}
In \Cref{thm:squared-linear} and \Cref{thm:convex-linear}, the action set $A\subseteq \bR$ is an interval, and the objective $f_0$ and the constraints $f_j$ are convex group objective/constraints. That is, for every $j\in \{0\}\cup J$, there exists $f_j':[t]\times A\times \{0,1\}\to \bR$ such that $f_j(x,a,y) = f_j'(g(x),a,y)$ for every $(x,a,y)\in X\times A\times \{0,1\}$, and the function $f_j'(i,\cdot,y)$ is convex for every $i\in [t]$ and $y\in \{0,1\}$. Moreover, the class $\cC_p$ is the class $\cC_p(g)$ in \Cref{def:c-p}, i.e., $\cC_p$ consists of all functions $c:X\to A$ such that there exists $\tau:[t]\times [0,1]\to A$ satisfying $c(x) = \tau(g(x),p(x))$ for every $x\in X$. Thus, \eqref{eq:optimize} becomes the following equivalent problem:
\begin{align}\label{eq:deterministic-optimize}
    \minimize_{\tau:[t]\times [0,1]\to A} \quad & \E_{(x,y)\sim\cD_p}f_0'(g(x), \tau(g(x),p(x)),y) \\
    \textnormal{s.t.}\quad & \E_{(x,y)\sim\cD_p}f_j'(g(x), \tau(g(x),p(x)),y) \leq \varepsilon/3 \quad \text{for every }j\in J.\nonumber
\end{align}
Let $V:=\mathsf{range}(p)$ denote the range of $p$. Since the functions $c\in \cC$ in \Cref{thm:squared-linear} and \Cref{thm:convex-linear} output bounded values $c(x)\in A = [0,1]$, we can always make sure that $V$ is finite and has size $O(1/\varepsilon')$ when we require $p$ to be $(\cC,g,\varepsilon')$-multiaccurate and/or $(\cC,g,\varepsilon')$-multicalibrated because discretizing the values $p(x)$ to multiples of $\varepsilon'/2$ can only increase the group multiaccuracy/multicalibration error by at most $\varepsilon'/2$.
Let $\prob(i,v,b)$ denote $\Pr_{(x,y)\sim\cD_p}[g(x) = i,p(x) = v, y = b]$. The optimization problem \eqref{eq:deterministic-optimize} above is equivalent to
\begin{align}\label{eq:deterministic-optimize-1}
    \minimize_{\tau:[t]\times V\to A} \quad & \sum_{i\in [t]}\sum_{v\in V}\sum_{b\in \{0,1\}}\prob(i,v,b)f_0'(i,\tau(i,v),b) \\
    \textnormal{s.t.}\quad & \sum_{i\in [t]}\sum_{v\in V}\sum_{b\in \{0,1\}}\prob(i,v,b)f_j'(i,\tau(i,v),b) \leq \varepsilon/3 \quad \text{for every }j\in J.\nonumber
\end{align}
Suppose for now that we know the probabilities $\prob(i,v,b)$. The optimization problem \eqref{eq:deterministic-optimize-1} above is a convex program with size $O(t\,|V|\cdot|J|)$ and thus can be solved efficiently assuming that we can efficiently compute $f'$ and its sub-gradient.
When we do not know $\prob(i,v,b)$, we can estimate it to sufficient accuracy using i.i.d.\ data points from $\cD_p$, which can be generated using unlabeled data points from $\cD$.
By standard concentration results (e.g.\ \Cref{claim:est-2}), using $n = O(\varepsilon_1^{-2}(|V|t + \log(1/\delta)))$ data points we can compute an estimator $\est(i,v,b)$ for $\prob(i,v,b)$ such that with probability at least $1-\delta$, 
\[
\sum_{i\in [t]}\sum_{v\in V}\sum_{b\in \{0,1\}} |\est(i,v,b) - \prob(i,v,b)| \le \varepsilon_1/3.
\]
In \Cref{thm:general}, the action set $A$ is a finite set, and the objective $f_0$ and the constraints $f_j$ are group objective/constraints. The class $\cC_p$ is the class $\cC_p^{\mathsf{rand}}(g)$, i.e., $\cC_p$ consists of all functions $c:X\to A$ such that there exists $\tau:[t]\times [0,1]\to A$ satisfying $c(x) = \tau(g(x),p(x))$ for every $x\in X$. Thus, \eqref{eq:optimize} becomes the following equivalent problem:
\begin{align}\label{eq:randomized-optimize}
    \minimize_{\tau:[t]\times [0,1]\to \Delta_A} \quad & \E_{(x,y)\sim\cD_p}\E_{a\sim \tau(g(x),p(x))}f_0'(g(x), a,y) \\
    \textnormal{s.t.}\quad & \E_{(x,y)\sim\cD_p}\E_{a\sim\tau(g(x),p(x))}f_j'(g(x), a,y) \leq \varepsilon/3 \quad \text{for every }j\in J.\nonumber
\end{align}
Defining $V$ and $\prob(i,v,b)$ as before and using $\tau'(i,v,a)$ to denote the probability mass on $a\in A$ in $\tau(i,v)$, the optimization problem \eqref{eq:randomized-optimize} above is equivalent to the following:
\begin{align}\label{eq:randomized-optimize-1}
    \minimize_{\tau':[t]\times V\times A\to \bR} \quad & \sum_{i\in [t]}\sum_{v\in V}\sum_{b\in \{0,1\}}\sum_{a\in A}\prob(i,v,b)\tau'(i,v,a)f_0'(i,a,b) \\
    \textnormal{s.t.}\quad & \sum_{i\in [t]}\sum_{v\in V}\sum_{b\in \{0,1\}}\sum_{a\in A}\prob(i,v,b)\tau'(i,v,a)f_j'(i,a,b) \leq \varepsilon/3, && \forall j\in J,\nonumber\\
    & \sum_{a\in A}\tau'(i,v,a) = 1, && \forall (i,v)\in [t]\times V,\nonumber \\
    & \tau'(i,v,a) \ge 0, && \forall (i,v,a)\in [t]\times V\times A.\nonumber
\end{align}
This optimization problem \eqref{eq:randomized-optimize-1} is a linear program of size $O(t\,|V|\cdot|A|\cdot|J|)$ and thus can be solved efficiently.

\section{Counterexamples}
\label{sec:counterexamples}
\subsection{Group Multiaccuracy is Necessary}\label{sec:counter-ex}
We show that the group multiaccuracy and group calibration assumptions in \Cref{thm:squared-linear} cannot be replaced by standard (non-group-wise) multicalibration.
\begin{claim}
Let $A = [0,1]$ be an action set.
There exists a non-empty set $X$ over individuals, a group partition function $g:X\to [t]$, a distribution $\cD$ over $X\times\{0,1\}$, a task $T$, a class $\cC$ of functions $c:X\to A$, a predictor $p:X\to [0,1]$ with the following properties. The task $T$ has the $\ell_1$ objective $f_0(x,a,y) = |a - y|$ and linear constraints (as in \eqref{eq:linear-constraint}). The predictor $p$ belongs to $\mc(\cC,0)\cap \cali(0)$. However, $p$ is not a $(\{T\},\cC,\cC_p(g),\varepsilon)$-omnipredictor for sufficiently small $\varepsilon > 0$.
\end{claim}
\begin{proof}
We assume that $X = \{x_1,x_2,x_3,x_4\}$ and $(x,y)\sim \cD$ can be sampled by first  drawing $x$ from the uniform distribution over $X$, and then drawing $y\sim \ber(p^*(x))$ for
\begin{align*}
    p^*(x)=\begin{cases}
     0.5,  &\text{if } x=x_1,\\
     0.5, &\text{if } x=x_2,\\
     0, &\text{if } x=x_3,\\
     1,  &\text{if } x=x_4.\\
    \end{cases}
\end{align*}

The function class $\mathcal{C}$ consists of a single function $c$ defined by
\begin{align*}
    c(x) = \begin{cases}
    0.75, \quad &x= x_1,\\
    0.25, \quad &x=x_2,\\
    0, \quad &x\in \{x_3,x_4\}.
    \end{cases}
\end{align*}
The groups are defined by
\begin{align*}
    g(x) = \begin{cases}
   1,\quad &x\in \{x_1,x_3\},\\
    2, \quad &x\in \{x_2,x_4\}.\\    
    \end{cases}
\end{align*}
The constraints $f_j$ of the task $T$ are defined by 
\begin{align*}
    &f_1(x,a,y) = \one(i=1)0.375- \one(i=1)a\\
    &f_2(x,a,y) = -\one(i=1)0.375 + \one(i=1)a\\
    &f_3(x,a,y) = \one(i=2)0.125 - \one(i=2)a\\
    &f_4(x,a,y) = -\one(i=2)0.125 + \one(i=2)a\\
\end{align*}
That is, they require that $\E[c(x)|g(x)=1]=0.375, \E[c(x)|g(x)=2]=0.125$.
We can easily see that $c$ satisfies the constraint:
\begin{align*}
    &\E_{x}[c(x)|g(x)=1] = 0.75\cdot0.5 = 0.375,\\
    &\E_{x}[c(x)|g(x)=2] = 0.25\cdot0.5 = 0.125.
\end{align*}

We choose $p:\cX\rightarrow[0,1]$ to be the constant function satisfying $p(x) = 0.5$ for all $x\in\cX$. We show that $p\in \mc(\cC,0)\cap \cali(0)$. 
We start from calibration:
\[ \E_{(x,y)\sim\cD}[y] = 0.5 = \E_{(x,y)\sim\cD}[p(x)].\]
Now we show multicalibration with respect to $c\in \cC$:
\begin{align*}
    & \E_{(x,y)\sim\cD}[c(x)\cdot(y-p(x))]\\
    ={} & \E_{(x,y)\sim\cD}[c(x)\cdot(y-0.5)]\\
    ={} &0.25\paren{0.75 (0.5-0.5) + 0.25(0.5-0.5)+ 0\cdot(0-0.5)+0\cdot(1-0.5))}=0.
\end{align*}

The objective value of $c$ is:
\begin{align*}
    \beta^* &:= \opt_\cD(T,\cC,0)\\
    & = \E_{(x,y)\sim\cD}[f_0(i,c(x),y)]\\ & = 0.125\paren{\abs{1-0.75}+\abs{0-0.75}+\abs{1-0.25}+\abs{0-0.25}} + 0.25\paren{\abs{0,0}+ \abs{1,0}}\\ & =
    0.125\paren{2\cdot 0.25 + 2\cdot 0.75} + 0.25=0.25+0.25\\ &=0.5.
\end{align*}
Since $p$ is a constant function, any $c'\in \cC_p(g)$ must satisfy $c'(x_1) = c'(x_3)$ and $c'(x_2) = c'(x_4)$ because $g(x_1) = g(x_3)$ and $g(x_2) = g(x_4)$.
To satisfy the constraints up to a small error $\varepsilon$, $c'$ must be close to assigning $0.375$ to $x_1$ and $x_3$, and assigning $0.125$ to $x_2$ and $x_4$. 
We calculate the loss for this $c'$:
\begin{align*}
\E_{(x,y)\sim\cD}[f_0(i,c'(x),y)] = {} & 0.125\paren{\abs{1-0.375}+\abs{0-0.375}+\abs{1-0.125}+\abs{0-0.125}} \\&+ 0.25\paren{\abs{0-0.325}+ \abs{1-0.125}}\\
={}&0.125\paren{0.625+0.375+0.125+0.875} + 0.25\paren{0.375+0.875}\\= {} & 0.25+ 0.25\cdot1.25\\
 = {} & 0.5625\\
 > {} & \beta^*.
\end{align*}
This implies that for small enough $\varepsilon$, we have $\cC_p(g)\cap \sol_\cD(T,\beta^* + \varepsilon,\varepsilon) = \emptyset$, and thus $p$ cannot be a $(\{T\},\cC,\cC_p(g),\varepsilon)$-omnipredictor.
\end{proof}
\subsection{Group Level-Set Multiaccuracy is Necessary}\label{sec:level-set}
We show an example task with non-convex constraints and a non-special objective, and thus none of our \Cref{thm:squared-linear,thm:convex-linear,thm:variant} could be applied to the example. \Cref{thm:general} is applicable, but it requires group level-set multiaccuracy. Below we show that for this task group multicalibration is indeed not enough and the level-set variant is necessary to guarantee omniprediction.
\begin{claim}
Let $A = [0,1]$ be an action set.
There exists a non-empty set $X$ over individuals, a group partition function $g:X\to [t]$, a distribution $\cD$ over $X\times\{0,1\}$, a task $T$, a class $\cC$ of functions $c:X\to A$, a predictor $p:X\to [0,1]$ with the following properties. The task $T$ only has group constraints and objectives with $1$-bounded differences. The predictor $p$ belongs to $\grpmc_{\cD}(\cC,g,0)\cap \grpcal_\cD(g,0)$. However, $p$ is not a $(\{T\},\cC,\cC_p^{\mathsf{rand}}(g),\varepsilon)$-omnipredictor for sufficiently small $\varepsilon > 0$.
\end{claim}
\begin{proof}
Let $\cX = \{x_1,x_2,x_3\}$ and let $g:X\to [t]$ be the trivial group partition that assigns every individual $x\in X$ to the same group $g(x) = 1$. The distribution $\cD$ is defined by first choosing $x\in\cX$ uniformly at random, and then choosing $y\sim \ber(p^*(x))$ for
\begin{align*}
    p^*(x) = \begin{cases}
    0.25, \quad &x= x_1,\\
    1, \quad &x = x_2,\\
    0.25, \quad &x=x_3.
    \end{cases}
\end{align*}
The function class $\cC$ contains only a single function $\cC = \set{c}$ defined by:
\begin{align*}
    c(x) = \begin{cases}
    0.1, \quad &x = x_1,\\
    0.2, \quad &x = x_2,\\
    0.3, \quad &x = x_3.
    \end{cases}
\end{align*}
We choose the objective $f_0$ of $T$ to be the cubic loss: $f_0(x,a,y) = |a - y|^3$.
We choose the collection of constraints $f_j$ of $T$ to be
\begin{align*}
    &f_1(x,a,y)= \one(a=0.1) - \frac{1}{3}\\
    &f_2(x,a,y)= -\one(a=0.1) + \frac{1}{3}\\
    &f_3(x,a,y)= \one(a=0.2) - \frac{1}{3}\\
    &f_4(x,a,y)= -\one(a=0.2) + \frac{1}{3}\\
    &f_5(x,a,y)= \one(a=0.3) - \frac{1}{3}\\
    &f_6(x,a,y)= -\one(a=0.3) + \frac{1}{3}\\
\end{align*}
For an action function $c':X\to A$ to satisfy these constraints exactly, it must satisfy 
\[
\Pr_{(x,y)\sim\cD}[c'(x)=a]=1/3 \quad \text{for every }a\in\set{0.1,0.2,0.3}. 
\]
It is clear that the only function $c\in \cC$ satisfies the constraints. The objective value achieved by $c$ is
\begin{align*}
\beta^*:=\opt_\cD(T,\cC,0) ={} &     \E_{(x,y)\sim \cD}[f_0(x,c(x),y)]\\
 ={}& \sum_j\Pr[x= x_j]\paren{\E_{(x,y)\sim \cD}[y|x\in U_j]|1-c_j|^3 + (1 -\E_{(x,y)\sim \cD}[y|x\in U_j])|c_j|^3} \\
    ={}& \frac{1}{3}\paren{\frac{1}{4}(0.9)^3 + \frac{3}{4}(0.1)^3 +1(0.8)^3 + 0(0.2)^3 +\frac{1}{4}(0.7)^3 + \frac{3}{4}(0.3)^3}\\
    ={}&0.267.
\end{align*}

The predictor $p:\cX\rightarrow[0,1]$ defined by $p(x)=0.5$ for all $x\in \cX$. We show that $p\in \grpmc_{\cD}(\cC,g,0)\grpcal_\cD(g,0)$.
We show it, starting from calibration:
\[ \E_{(x,y)\sim\cD}[y] = 0.5 = \E_{(x,y)\sim\cD}[p(x)]. \]
For group multicalibration with respect to $c\in\cC$:
\begin{align*}
    \E_{(x,y)\sim\cD}\left[c(x)\left(y - p(x)\right)\right] &=\frac{1}{3}\paren{-\frac{1}{10}\cdot\frac{1}{4}+\frac{2}{10}\cdot\frac{1}{2}-\frac{3}{10}\cdot\frac{1}{4}} = 0
\end{align*}

Since both $p$ and $g$ are constant functions, any $c'\in \cC_p^{\mathsf{rand}}(g)$ has to give all $x\in X$ the same distribution $c(x)$ of actions. %
To satisfy the constraints up to a small error $\varepsilon$, $c'(x)$ must be close to the uniform distribution over $\{0.1, 0.2, 0,3\}$ for every $x$.
When $c'(x)$ is this uniform distribution for every $x$, we have
\begin{align*}
    \E_{(x,y)\sim \cD}\E_{a\sim c(x)}[f_0(x,a,y)] ={}&
    \sum_{b\in\set{0,1},a\in\set{0.1,0.2,0.3}} \Pr_{(x,y)\sim\cD}[y=b,c(x)=a]\abs{y-a}^3\\
    ={}& \frac{1}{2}\cdot\frac{1}{3}\paren{(0.9)^3 + (0.1)^3+(0.8)^3 + (0.2)^3+(0.7)^3 + (0.3)^3}\\
    ={}& 0.27\\
    > {} & \beta^*.
\end{align*}
Therefore, for small enough $\varepsilon > 0$, we have $\cC_p^{\mathsf{rand}}(g)\cap \sol(T,\beta^*+\varepsilon,\varepsilon) = \emptyset$, and thus $p$ cannot be a $(\{T\},\cC,\cC_p^{\mathsf{rand}}(g),\varepsilon)$-omnipredictor.
\end{proof}
\section{Helper Claims}
The following claim is a standard result (see e.g.\ \citep[Theorem 1]{canonne2020short}):
\begin{claim}
\label{claim:est-2}
Let $Z$ be a non-empty set partitioned into $Z\sps 1,\ldots, Z\sps m$.
For $\varepsilon,\delta\in (0,1/2)$ and an integer $n \ge W(\varepsilon^{-2}(m + \log(1/\delta)))$ for a sufficiently large absolute constant $W > 0$, let $z_1,\ldots,z_n\in Z$ be $n$ data points drawn i.i.d.\ from any distribution $\cD$ over $Z$. Then with probability at least $1-\delta$, the following inequality holds:
\[
\sum_{j=1}^m\left|\frac 1n\sum_{i=1}^n\one(z_i\in Z\sps j) - \Pr_{z\sim \cD}[z\in Z\sps j]\right| \le \varepsilon.
\]
\end{claim}
\end{document}